\newif\ifnoarxiv

\ifnoarxiv
\documentclass[twoside,11pt]{article}
\usepackage{ntheorem} 
\usepackage[nohyperref]{jmlr2e}
\usepackage[colorlinks=false,hidelinks]{hyperref}

\else 
\documentclass[a4paper,11pt]{article}
\usepackage[top=30mm,left=29mm,right=29mm]{geometry}
\usepackage{amssymb}
\usepackage{amsthm}
\usepackage[round]{natbib}
\usepackage{hyperref}
\newcommand{\lcucolor}{0.1,0.1,0.9}
\hypersetup{colorlinks,
            linktocpage=true,
            linkcolor=[rgb]{\lcucolor},
            citecolor=[rgb]{\lcucolor},
            urlcolor=[rgb]{\lcucolor}}
\fi

\usepackage{amsmath}
\usepackage{enumitem}
\usepackage{xspace}
\usepackage{bm} 
\usepackage[capitalize]{cleveref}

\ifnoarxiv
\else
\newtheorem{theorem}{Theorem}
\newtheorem{lemma}[theorem]{Lemma}
\newtheorem{corollary}[theorem]{Corollary}

\fi

\usepackage[disable]{todonotes}
\newcommand{\gaussian}{Gaussian\xspace}

\newcommand{\subgaussian}{sub-Gaussian\xspace}

\newcommand{\subexponential}{subexponential\xspace}

\newcommand{\crefi}[2]{\cref{#1}\ref{#1:#2}}
\newcommand{\cond}[1]{\eqref{cond:#1}}

\newcommand{\defeq}{\doteq} 
\newcommand{\T}{\top} 
\newcommand{\argmin}{\mathop{\mathrm{argmin}}}
\newcommand{\argmax}{\mathop{\mathrm{argmax}}}
\newcommand{\Ordo}{O}

\newcommand{\setN}{\mathbb{N}} 
\newcommand{\setR}{\mathbb{R}} 
\newcommand{\setX}{\mathbb{X}} 
\newcommand{\setF}{\mathcal{F}} 
\newcommand{\setK}{\mathcal{K}} 
\newcommand{\setW}{\mathbb{W}} 
\newcommand{\setS}{\mathcal{S}} 

\newcommand{\ind}{\mathbb{I}} 
\newcommand{\youngfunsym}{\Psi} 
\newcommand{\youngfun}[1][q]{\youngfunsym_{#1}} 
\newcommand{\distF}{\psi} 
\newcommand{\proc}{\Lambda} 
\newcommand{\procl}{\Gamma} 
\newcommand{\procmod}{\tau} 
\newcommand{\Nc}{\mathcal{N}} 
\newcommand{\NcF}{\Nc_{\distF}} 
\newcommand{\He}{\mathcal{H}} 
\newcommand{\HeF}{\He_{\distF}} 
\newcommand{\gkf}{g_k(f)} 
\newcommand{\gmf}{g_m(f)} 
\newcommand{\gf}{g_f} 
\newcommand{\maxgf}{\gf^*}
\newcommand{\maxgh}{g_h^*}

\newcommand{\setFXR}{\{\setX\to\setR\}} 
\newcommand{\setFlinUc}{\setF_{\textrm{aff}}} 

\newcommand{\setFlinL}[1]{\setFlinUc^{#1}} 
\newcommand{\setFlin}{\setFlinL{L}} 
\newcommand{\setFlinmuL}[1]{\setFlinUc^{#1,\mu}}
\newcommand{\setFlinmu}{\setFlinmuL{L}}
\newcommand{\setFlinDnL}[1]{\setFlinUc^{#1,n}}
\newcommand{\setFlinDn}{\setFlinDnL{L}}
\newcommand{\hL}{\hat{L}}

\newcommand{\Id}{I} 

\newcommand{\norm}[1]{\left\lVert#1\right\rVert}
\newcommand{\smallnorm}[1]{\lVert#1\rVert} 
\newcommand{\enorm}[1]{\norm{#1}} 
\newcommand{\esmallnorm}[1]{\smallnorm{#1}} 
\newcommand{\onorm}[2][q]{\norm{#2}_{\youngfunsym_{#1}}} 
\newcommand{\osmallnorm}[2][q]{\smallnorm{#2}_{\youngfunsym_{#1}}} 

\renewcommand{\vec}[1]{{\boldsymbol{#1}}}
\newcommand{\vzero}{\vec{0}} 
\newcommand{\vx}{\vec{x}}

\newcommand{\va}{\vec{a}}
\newcommand{\vb}{\vec{b}}

\newcommand{\E}{\mathbb{E}} 
\renewcommand{\P}{\mathbb{P}} 
\newcommand{\K}{\mathbb{K}} 
\newcommand{\Kz}{\K_0} 
\newcommand{\as}{\textrm{a.s.}\xspace} 
\newcommand{\iid}{\textrm{i.i.d.}\xspace} 
\newcommand{\gaussiandistr}{\mathcal{N}} 

\newcommand{\X}{\mathcal{X}}
\newcommand{\vX}{\vec{\X}}
\newcommand{\Y}{\mathcal{Y}}
\newcommand{\Z}{\mathcal{Z}}
\newcommand{\vZ}{\vec{\Z}}
\newcommand{\W}{\mathcal{W}}
\newcommand{\vW}{\vec{\W}}
\newcommand{\V}{\mathcal{V}}

\newcommand{\vXbar}{\mkern 5mu\overline{\mkern-5mu\vX\mkern-2mu}\mkern 2mu}
\newcommand{\Ybar}{\mkern 3mu\overline{\mkern-3mu\Y\mkern-1mu}\mkern 1mu}

\newcommand{\radFo}{r_0} 
\newcommand{\momr}{r} 
\newcommand{\moms}{\theta} 
\newcommand{\oqmod}{S} 
\newcommand{\procmodB}{T} 
\newcommand{\apxB}{B_*} 
\newcommand{\Bc}{C} 
\newcommand{\Kn}{K_n} 
\newcommand{\scl}{\eta} 
\newcommand{\tn}{t} 
\newcommand{\Rtn}{B} 
\newcommand{\eigmu}{\eta_\mu} 

\newcommand{\loss}{\ell} 
\newcommand{\losssq}{\loss_{\textrm{sq}}} 
\newcommand{\lossce}{\loss_{\textrm{ce}}} 

\newcommand{\dr}{\mu} 
\newcommand{\setFr}{\setF_*} 
\newcommand{\fr}{f_*} 
\newcommand{\Dn}{\mathcal{D}_n} 
\newcommand{\setFn}{\setF_n} 
\newcommand{\fe}{f_n} 
\newcommand{\err}{\alpha} 
\newcommand{\pen}{\beta} 
\newcommand{\setFi}{\hat\setF_n} 
\newcommand{\setFo}{\setF} 
\newcommand{\setFor}{\setFo(\radFo)} 
\newcommand{\risk}{R_\dr} 
\newcommand{\riskn}{R_n} 
\newcommand{\erisk}{L_{\dr}} 
\newcommand{\eriskn}{L_n} 
\newcommand{\Lmod}{G} 
\newcommand{\Zrisk}{\Z} 
\newcommand{\Zriski}{\Z_i} 
\newcommand{\hn}{h_n} 
\newcommand{\Wf}{\Delta_{f,\fr}} 
\newcommand{\Zf}{\Z_f} 
\newcommand{\Vf}{\V_f} 
\newcommand{\hy}{\hat{y}} 
\newcommand{\cgamma}{T_{\ln}} 
\newcommand{\BX}{\rho} 
\newcommand{\BY}{\sigma} 
\newcommand{\fcov}{\Sigma} 
\newcommand{\fcovh}{\hat\Sigma} 
\newcommand{\rlip}{\tau} 
\newcommand{\rlipLog}{\rlip_{\ln}} 

\newcommand{\ripar}{\lambda} 
\newcommand{\riva}{\va_{n,\ripar}} 
\newcommand{\rife}{\fe^\ripar} 
\newcommand{\penLtwo}{\pen_{\ripar}} 
\newcommand{\Ls}{L_*} 

\newcommand{\setM}{\mathbb{M}} 
\newcommand{\setMbg}{\setM_{\textrm{bg}}} 
\newcommand{\setMls}{\setM_{\textrm{subg}}^{\BX,\BY,d}} 

\newcommand{\todob}[2][]{\todo[color=blue!20!white,#1]{BG: #2}}

\ifnoarxiv
\jmlrheading{1}{2016}{1-48}{4/00}{10/00}{G\'abor Bal\'azs, Csaba Szepesv\'ari, and Andr\'as Gy\"orgy}

\ShortHeadings{Chaining Bounds for Empirical Risk Minimization}{Bal\'azs, Szepesv\'ari, and Gy\"orgy}

\firstpageno{1}
\fi

\begin{document}

\title{Chaining Bounds for Empirical Risk Minimization}

\ifnoarxiv
\else
\author{G\'abor Bal\'azs and Csaba Szepesv\'ari \\
        Department of Computing Science,
        University of Alberta, \\
        Edmonton, Alberta, T6G 2E8, Canada \\
        \ \\
        Andr\'as Gy\"orgy \\
        Department of Electrical and Electronic Engineering, \\
        Imperial College London,
        London, SW7 2BT, United Kingdom}
\fi

\maketitle

\begin{abstract}
  This paper extends the standard chaining technique to prove excess risk
  upper bounds for empirical risk minimization with random design settings
  even if the magnitude of the noise and the estimates is unbounded.
  The bound applies to many loss functions besides the squared loss,
  and scales only with the \subgaussian or \subexponential parameters
  without further statistical assumptions
  such as the bounded kurtosis condition over the hypothesis class.
  A detailed analysis is provided for slope constrained and penalized
  linear least squares regression with a \subgaussian setting,
  which often proves tight sample complexity bounds up to logartihmic factors.
\end{abstract}

\section{Introduction}
\label{sec:intro}

This paper extends the standard chaining technique
\citep[e.g.,][]{Pollard1990,Dudley1999,Gyorfi2002,BoucheronLugosiMassart2012}
to prove high-probability excess risk upper bounds
for empirical risk minimization (ERM)
for random design settings
even if the magnitude of the noise and the estimates is unbounded.
Our result (\cref{thm:ermub}) covers bounded settings
\citep{BartlettBousquetMendelson2005,Koltchinskii2011},
extends to \subgaussian or even \subexponential noise
\citep{VanDeGeer2000,GyorfiWegkamp2008},
and handles hypothesis classes with unbounded magnitude
\citep{LecueMendelson2013,Mendelson2014,LiangRakhlinSridharan2015}.
Furthermore, it applies to many loss functions besides the squared loss,
and does not need additional statistical assumptions
such as the bounded kurtosis of the transformed covariates
over the hypothesis class,
which prevent the latest developments to provide tight
excess risk bounds for many \subgaussian cases (\cref{sec:review}).

To demonstrate the effectiveness of our method for such unbounded settings,
we use our general excess risk bound (\cref{thm:ermub})
to provide a detailed analysis for linear least squares estimators
using quadratic slope constraint and penalty
with \subgaussian noise and domain
for the random design, nonrealizable setting (\cref{sec:linear}).
Our result for the slope constrained case
extends Theorem~A of \citet{LecueMendelson2013}
and nearly proves the conjecture of \citet{Shamir2015},
while our treatment for the penalized case (ridge regression)
is comparable to the work of \citet{HsuKakadeZhang2014}.

The rest of this section introduces our notation through the formal definition
of the regression problem and ERM estimators (\cref{sec:regprb}),
and discusses the limitations of current excess risk upper bounds
in the literature (\cref{sec:review}).
Then, we provide our main result in \cref{sec:ubound} to upper bound
the excess risk of ERM estimators, and discuss its properties
for various settings including many loss functions besides the squared loss.
Next, \cref{sec:linear} provides a detailed analysis
for linear least squares estimators
including the slope constrained case (\cref{sec:linear-lasso})
and ridge regression (\cref{sec:linear-ridge}).
Finally, \cref{sec:uboundprf} proves our main result (\cref{thm:ermub}).

\subsection{Empirical risk minimization}
\label{sec:regprb}

For the formal definition of a regression problem,
consider a probability distribution $\mu$
over some set $\setX \times \setR$
with some \emph{domain} $\setX$ being a separable Hilbert space,%
\footnote{
  All sets and functions considered are assumed to be measurable as necessary.
  To simplify the presentation, we omit these conditions by noting here that
  all the measurability issues can be overcome
  using standard techniques as we work with separable Hilbert spaces
  \citep[e.g.,][Chapter~5]{Dudley1999}.
}
a \emph{loss function} $\loss : \setR \times \setR \to [0,\infty)$,
and a \emph{reference class}
$\setFr \subseteq \setFXR \defeq \{f \,|\, f : \setX \to \setR\}$.

The task of a regression estimator is to produce a function
$f \in \setFXR$ based on a \emph{training sample}
$\Dn \defeq \{(\vX_1,\Y_1),\ldots,(\vX_n,\Y_n)\}$
of $n \in \setN$ pairs $(\vX_i,\Y_i) \in \setX\times\setR$
independently sampled from $\mu$ (in short $\Dn \sim \mu^n$),
such that the prediction error,
$\loss(\Y,f(\vX))$, is ``small'' on a new instance $(\vX,\Y) \sim \mu$
with respect to $\loss$.

The \emph{risk} of function $f \in \setFXR$
is defined as $\risk(f) \defeq \E[\loss(\Y,f(\vX))]$
and the cost of using a fixed function $f$
is measured by the \emph{excess risk}
with respect to $\setFr$:
\[ \erisk(f,\setFr) \defeq \risk(f) - \min_{g\in\setFr}\risk(g) \,. \]
We also use the notation $\erisk(f,g) \defeq \risk(f) - \risk(g)$
for any $f, g \in \setFXR$, hence we can write
$\erisk(f,\setFr) = \erisk(f,\fr)$
for any $\fr \in \argmin_{g\in\setFr}\risk(g)$.%
\footnote{A straightforward limiting argument can be used
          if the minimums are not attained for $\fr$.}

An \emph{estimator} $\hn$ is a
sequence of functions $h \defeq (\hn)_{n\in \setN}$, where
$\hn : (\setX \times \setR)^n \to \setFXR$
maps the data $\Dn$ to an estimate $\fe \defeq \hn(\Dn)$.
These estimates lie within some \emph{hypothesis class}
$\setFn \subseteq \setFXR$, that is $\fe \in \setFn$,
where $\setFn$ might depend on the random sample $\Dn$.

Then, for a regression problem specified by $(\loss,\mu,\setFr)$,
the goal of an estimator $\hn$ is to produce estimates which
minimize the excess risk $\erisk(\fe,\setFr)$
with high-probability or in expectation,
where the random event is induced by the random sample $\Dn$
and the possible randomness of the estimator $\hn$.

In this paper, we consider ERM estimators.
Formally, $\fe$ is called an
\emph{$\err$-approximate \mbox{$\pen$-penalized} ERM estimate}
with respect to the class $\setFn$,
in short $\fe \in (\err,\pen)$-ERM($\setFn$),
when $\fe \in \setFn$ and
\begin{equation}
\label{eq:erm}
  \riskn(\fe) + \pen(\fe)
  \le \inf_{f\in\setFn}\riskn(f) + \pen(f) + \err \,,
\end{equation}
where $\riskn(f) \defeq \frac{1}{n}\sum_{i=1}^n\loss(\Y_i,f(\vX_i))$
      is the \emph{empirical risk} of function $f \in \setFXR$,
      \mbox{$\pen : \setFn \to [0,\infty)$} is a \emph{penalty function}
and $\err \ge 0$ is an \emph{error term}.
All $\err$, $\pen$, and $\setFn$ might depend on the sample $\Dn$.
When the penalty function is zero (that is $\pen \equiv 0$),
we simply write $\fe \in \alpha$-ERM($\setFn$).
If both $\err = 0$ and $\pen \equiv 0$, we say $\fe \in$ ERM($\setFn$).

\subsection{Limitations of current methods}
\label{sec:review}

Now we provide a simple regression problem class for which
we are not aware of any technique in the literature
that could provide a tight excess risk bound
up to logarithmic factors for empirical risk minimization.
Consider the following problem set:
\begin{equation*} \begin{split}
  \setMbg
  \defeq \big\{ \dr \,\big|\, & (\vX,\Y) \sim \dr
                          ,\, \vX = [\W\,\,\Z\,\,1]^\T
                              \in \setX \subset \setR^3
                ,\, \\ & \hspace{1mm}
                \W \sim \gaussiandistr(0,1) ,\,
                \P\{\Z = -p\} = 1-p ,\, \P\{\Z = 1-p\} = p ,\, p \in (0,1)
                ,\, \\ & \hspace{1mm}
                          \Y = \fr(\vX) + 1/2 + \xi~\as,\,
                          \fr(\vx) = [0\,\,0\,\,1/2]\vx
                          ,\,\, \vx \in \setX
                          ,\,\, \xi \sim \gaussiandistr(0,1)
         \big\}
  \,,
\end{split} \end{equation*}
where $\setX \defeq \setR \times [-1,1] \times \{1\}$,
$\as$~stands for almost surely,
and $\gaussiandistr(0,1)$ denotes the centered \gaussian distribution
with unit variance.
Furthermore, define the linear function class
\mbox{$\setF = \{f\,|\,f(\vx) = \va^\T\vx,\,\enorm{\va} \le 1/2\}$},
and consider the squared loss $\loss = \losssq$ defined as
\mbox{$\losssq(y,\hy) \defeq |y-\hy|^2$} for all $y, \hy \in \setR$.
Notice that $\fr \in \setF$
and $\fr = \argmin_{f\in\setF}\risk(f)$ for all $\dr \in \setMbg$.
Then, we discuss various techniques from the literature which aim to
bound the ``performance'' of an estimate $\fe \in$ ERM($\setF$).

Because here we have a random design setting,
the results of \citet[Theorems~9.1 and 9.2]{VanDeGeer2000} do not apply.
Moreover, the regression function cannot be represented by the class $\setF$,
so the methods of \citet[Theorem~11.3]{Gyorfi2002},
and \citet[Corollary~1]{GyorfiWegkamp2008}
do not provide an excess risk bound.

As the domain $\setX$ is unbounded,
so does the range of any nonzero function in $\setF$.
Additionally, the squared loss $\losssq$ is neither Lipschitz, nor bounded
on the range of response $\Y$ which is the whole real line $\setR$.
Hence, the techniques including
\citet[Corollary~5.3]{BartlettBousquetMendelson2005},
\citet[Theorem~5.1]{Koltchinskii2011},
\citet[Theorem~6]{MehtaWilliamson2014},
\citet[Theorem~14 with Proposition~4]{GrunwaldMehta2016}
fail to provide any rate for this case.

We also mention the work of \citet[Theorem~3.2]{VaartWellner2011},
which, although works for this setting,
can only provide an $\Ordo(\sqrt{\ln(n)/n})$ rate
for sample size $n$, which
can be improved to $\Ordo(\ln(n)/n)$
by our result (\cref{thm:ermub,thm:cond-mom}).

Next, denote the \emph{kurtosis about the origin} 
by $\Kz[\W] \defeq \E[\W^4]/\E[\W^2]^2$
for some random variable $\W$,
and consider the recent developments
of \citet[Theorem~A]{LecueMendelson2013},
and \citet[Theorem~7]{LiangRakhlinSridharan2015}.
These results need that the kurtosis of the random variables
$\Wf \defeq f(\vX)-\fr(\vX)$ is bounded
for any $f\in\setF \setminus \{\fr\}$.
However, observe that $\Kz[\Wf] \ge (1-p)^2/p$ for any function
$f(\vx) = [0\,\,a\,\,0]\vx$ with $0 < |a| \le 1/2$,
which can be arbitrarily large as $p$ gets close to zero.

Finally, we mention the result of \citet[Theorem~2.2]{Mendelson2014},
which bounds the squared deviation of the ERM estimator $\fe$ and $\fr$,
that is $\E\big[|\fe(\vX)-\fr(\vX)|^2\big]$.
However, as pointed out by \citet[Section~1]{Shamir2015},
this can be arbitrarily smaller than the excess risk $\erisk(\fe,\fr)$
for functions $f_a(\vx) = [0\,\,0\,\,a]\vx$ with $|a| \le 1/2$,
which means that
\mbox{$\erisk(f_a,\fr)/\E\big[|f_a(\vX)-\fr(\vX)|^2\big] \to \infty$}
as $a \to 1/2$.

\subsection{Highlights of our technique}

Our excess risk bound builds on the development of inexact oracle inequalities
for ERM estimators \citep[e.g.,][Theorem~11.5]{Gyorfi2002},
which uses the decomposition
\begin{equation*} \begin{split}
  \erisk(\fe,\fr)
  &= \erisk(\fe,\fr) - c\eriskn(\fe,\fr) + c\eriskn(\fe,\fr)
  \\
  &\le \sup_{f\in\setFn} \big\{\erisk(f,\fr) - c\eriskn(f,\fr)\big\}
     + c \inf_{f\in\setFn}\eriskn(f,\fr)
  \,,
\end{split} \end{equation*}
with some $c > 1$. Then the random variables
$\erisk(f,\fr) - c\eriskn(f,\fr)$ for all~$f \ne \fr$, having a negative bias,
often satisfy a moment condition \cond{mom},
which we cannot guarantee for $c = 1$.
Using this moment condition, we can augment the chaining technique
\citep[e.g.,][Section~3]{Pollard1990} with an extra
initial step, which provides a new $\Ordo(1/n)$ term in the bound.
This new term can be balanced with the (truncated) entropy integral,
so tightening the bound significantly in many cases.

By defining $\fr$ as a reference function (instead of regression function),
the inexact oracle inequalities become exact when $\fr \in \setFn$.
In fact, the notion of exact and inexact becomes meaningless 
as long as the approximation error between $\setFn$ and $\fr$
is kept under control and incorporated into the bound
as it is often done for sieved estimators
\citep[e.g.,][Section~10.3]{VanDeGeer2000}.

To prove the moment condition \cond{mom}
for a reference function $\fr$ and a hypothesis class~$\setFn$,
we use Bernstein's inequality (\cref{thm:Bernstein})
with the Bernstein condition \eqref{eq:Bernstein}.
These tools are standard, however we have to use Bernstein's inequality
for the \subgaussian random variable $\Wf$
so that $\E[\Wf^2]$ appears in the bound.
A naive way to do this would require the kurtosis $\Kz[\Wf]$
to be bounded for all $f \in \setFn$, which cannot be guaranteed
in many cases (\cref{sec:review}).
Hence, we use a truncation technique (\cref{thm:onorm-mom})
that pushes the kurtosis bound $\sup_{f\in\setFn}\Kz[\Wf]$
under a logarithmic transformation,
which can be eliminated
by considering functions $f \in \setFn$
with excess risk $\erisk(f,\fr)$ bounded away from zero.

The Bernstein condition \eqref{eq:Bernstein} has been well-studied
for strongly-convex loss functions
\citep[e.g.,][Lemma~7]{BartlettJordanMcAuliffe2006},
by exploiting that strong-convexity provides
an upper bound to the quadratic function.
However, because it is enough for our technique to consider functions
with excess risk bouded away from zero, we can use the Bernstein condition
for any Lipschitz loss function (\cref{sec:ubound-lip})
by scaling its parameters depending on the sample size $n$
and balancing the appropriate terms in the excess risk bound (\cref{thm:ermub}).
In many cases, this provides an alternative way for deriving
excess risk bounds for other loss functions without using the entropy integral.

\section{Excess risk upper bound}
\label{sec:ubound}

Here we are going to state our excess risk upper bound for ERM estimators.

Our result requires a few conditions to be satisfied by the random variables
$\Zrisk(f,g) \defeq \loss(\Y,f(\vX))-\loss(\Y,g(\vX))$
with $f,g \in \setFXR$,
which are related to the excess risk
through $\erisk(f,g) = \E[\Zrisk(f,g)]$.
Similarly, we use the \emph{empirical excess risk} defined as
$\eriskn(f,g) \defeq \riskn(f) - \riskn(g)
                   = \frac1n \sum_{i=1}^n \Zriski(f,g)$,
where $\Zriski(f,g) \defeq \loss(\Y_i,f(\vX_i)) - \loss(\Y_i,g(\vX_i))$.

We also use \subexponential random variables ($d=1$)
and vectors $\vW \in \setR^d$ characterized by
the $\youngfun$-Orlicz norm with $q \ge 1$ defined as
$\onorm{\vW} \defeq \inf\{B > 0 : \E[\youngfun(\vW/B)] \le 1\}$,
where $\youngfun(\vx) \defeq e^{\enorm{\vx}^q}-1$, $\vx \in \setR^d$,
$\enorm{\cdot}$ is the Euclidean norm,
and $\inf\emptyset = \infty$.
The properties of random vectors $\vW$ with $\onorm{\vW} < \infty$
are reviewed in \cref{sec:random}.

Furthermore, we need covering numbers and entropies.
Let $(\setF,\distF)$ be a nonempty metric space and $\epsilon \ge 0$.
The set $\{f_1,\ldots,f_k\} \subseteq \setF$ is called an (internal)
\emph{$\epsilon$-cover} of $\setF$ under~$\distF$
if the $\distF$-balls of centers $\{f_1,\ldots,f_k\}$ and radius $\epsilon$
cover~$\setF$:
for any \mbox{$f \in \setF$}, $\min_{i=1,\ldots,k}\distF(f,f_i) \le \epsilon$.
The \emph{$\epsilon$-covering number} of $\setF$ under~$\distF$,
denoted by $\NcF(\epsilon,\setF)$,
is the cardinality of the $\epsilon$-cover with the fewest elements:
\begin{equation*}
    \NcF(\epsilon,\setF)
    \defeq \inf\Big\{k \in \setN
                     \,\big|\,
                     \exists f_1, \ldots, f_k \in \setF :
                     \sup_{f\in\setF}\,\min_{i=1,\ldots,k}
                                       \distF(f,f_i) \le \epsilon
               \Big\}
\end{equation*}
with $\inf\emptyset \defeq \infty$.
Further, the \emph{$\epsilon$-entropy} of $\setF$ under $\distF$
is defined as the logarithm of the covering number,
$\HeF(\epsilon,\setF) \defeq \ln\NcF(\epsilon,\setF)$.

Finally, our upper bound on the excess risk of ERM estimates
is the following:

\begin{theorem}
\label{thm:ermub}
  Consider a regression problem $(\loss,\dr,\setFr)$
  with an i.i.d.~training sample \mbox{$\Dn \sim \dr^n$}.
  Let $\setFn \subseteq \setFXR$ be a hypothesis class
  which might depend on the data $\Dn$,
  and let $\fe \in (\err,\pen)$-ERM($\setFn$).
  Further, let $\setFi, \setFo \subseteq \setFXR$ be two function classes,
  where $\setFi$ might depend on $\Dn$, but $\setFo$ might depend on the
  sample only through its size $n$.
  Finally, suppose that the following conditions hold
  for some metric $\distF : \setFo \times \setFo \to [0,\infty)$,
  $\gamma \in (0,1)$, $\fr \in \argmin_{f\in\setFr}\risk(f)$,
  and $\setFor \defeq \{f \in \setFo : \erisk(f,\setFr) > \radFo/n\}$
  for some $\radFo \ge 0$:
  \begin{enumerate}[label=($C_\arabic*$),ref=$C_\arabic*$]
  \itemsep0em
  \item \label{cond:enclose}
        the enclosement $\setFi \subseteq \setFn \subseteq \setFo$
        holds with probability at least $1-\gamma/4$,
  \item \label{cond:bound}
        there exists $\apxB \ge 0$ such that
        $\P\big\{\inf_{f\in\setFi}\eriskn(f,\fr) + \pen(f) + \err
                 \le \apxB\big\} \ge 1-\gamma/4$,
  \item \label{cond:lip}
        there exists $\Lmod : \setX \times \setR \to [0,\infty)$
        and $T \ge 0$ such that
        $\Zrisk(f,g) \le \Lmod(\vX,\Y) \, \distF(f,g)$~\as
        for all $f, g \in \setFor$, and
        $\P\big\{\E[\Lmod(\vX,\Y)]+\frac1n\sum_{i=1}^n\Lmod(\vX_i,\Y_i)
                 \le \procmodB\big\} \ge 1-\gamma/4$,
  \item \label{cond:psi}
        there exists $\oqmod \in (0,\infty]$
        and $q \in \{1,2\}$
        \todob{Would be nice with $q \ge 1$ in general.
               But extending \cref{thm:onorm-sum-indept} is not trivial.}
        such that
        $\onorm{\E[\Zrisk(f,g)] - \Zrisk(f,g)} \le \oqmod \, \distF(f,g)$
        holds for all $f, g \in \setFor$,
  \item \label{cond:mom}
        there exist $\momr \in (0,1]$ and $\moms > 0$ such that
        \[ \sup_{f\in\setFor}\E\big[e^{(\momr/\moms)\E[\Zrisk(f,\fr)]
                                     - (1/\moms)\Zrisk(f,\fr)}\big] \le 1
           \,. \]
  \end{enumerate}
  Then for all $\epsilon \ge \delta \ge 0$,
  we have with probability at least $1 - \gamma$ that
  \begin{equation*}
    \erisk(\fe,\setFr)
    \le \frac1\momr 
        \bigg( \hspace{-1mm}
               \moms
               \frac{\HeF\big(\epsilon,\setFor\hspace{-0.5mm}\big)
               \hspace{-1mm}+\hspace{-1mm}
               \ln(\frac4\gamma)}{n}
             + \frac{32\oqmod}{\sqrt{n}}
               \hspace{-1mm}
               \int_{\delta}^{\epsilon} \hspace{-2mm}
               \sqrt[q]{2\HeF\big(z,\setFor\hspace{-0.5mm}\big) 
                        \hspace{-1mm}
                        + \hspace{-1mm}
                        \ln\hspace{-1mm}\Big(\frac{32\epsilon}
                                                  {z\gamma}\Big)}
               dz
             + 8\delta\procmodB + \apxB
               \hspace{-1mm}
        \bigg)
        + \frac{\radFo}{n}
    .
  \end{equation*}
  Furthermore, the result holds without \cond{psi}, that is
  using $S = \infty$, $\epsilon = \delta$, and $\infty\cdot0 = 0$.
\end{theorem}
The proof of \cref{thm:ermub} is presented in \cref{sec:uboundprf}.

We point out that \cond{enclose} disappears when one sets
$\setFi = \setFn = \setFo$ as it is usually done in the literature.
However, this is an implicit assumption that either $\E\vX$ is small
enough to be negligible (i.e., $\E\vX \approx \vzero$),
or equivalently the estimator knows the value of $\E\vX$.
By choosing the sets $\setFi, \setFn, \setFo$ to be slightly different,
\cref{thm:ermub} covers the practical case when an estimator approximates
$\E\vX$ and $\E\Y$ by their empirical versions
$\vXbar \defeq \frac1n\sum_{i=1}^n\vX_i$
and $\Ybar \defeq \frac1n\sum_{i=1}^n\Y_i$, respectively,
so uses a data-dependent hypothesis class $\setFn$.

Notice that if $\fr \in \setFi$, then \cond{bound} reduces to bounding the
penalty and error terms, that is proving
$\inf_{f\in\setFi}\pen(f) + \err \le \apxB$ with
probability at least $1-\gamma/4$.
When $\fe \in$ ERM($\setFn$), which is a usual setting in the literature,
\cond{bound} is immediately satisfied by $\apxB = 0$.
In this case \cref{thm:ermub} is an \emph{exact oracle inequality}
\citep[e.g.,][Eq.~1.1]{LecueMendelson2013}.

Furthermore, observe that \cref{thm:ermub} uses metric $\distF$ for the
entropy $\HeF(\cdot,\setFor)$, which is related to the loss function $\loss$
through \cond{lip} and \cond{psi}. This allows us to apply the result to
estimates $\fe$ with unbounded magnitude,
which can be parametrized by some bounded space.
In such case $\distF$ is defined on the bounded parameter space
which keeps the entropy finite.

In the following sections we provide a detailed analysis
for the \emph{moment condition} \cond{mom},
showing that it holds for many practical settings
and loss functions besides the squared loss.
We note that \cond{mom} is very similar to
the \emph{stochastic mixability condition} of
\citet[Section~2.1]{MehtaWilliamson2014},
which is equivalent to \cond{mom} with $\momr = 0$ and $\radFo = 0$.

Finally, we mention that if the conditions of \cref{thm:ermub} hold
for all $\gamma \in (0,1)$,
we can transform the result to an expected excess risk bound.
To see this, suppose that $\P\{\erisk(\fe,\setFr)
                           \le b + \frac{c\,\ln^m(1/\gamma)}{n}\} \ge 1-\gamma$
holds for all $\gamma \in (0,1)$, some $b, c > 0$, and some $m \in \setN$.
Then setting $\gamma = e^{-(n\,t/c)^{1/m}}$ for any $t > 0$, we get
\begin{equation} \begin{split}
\label{eq:experisk}
  \E[\erisk(\fe,\setFr)] - b
 &\le \E\big[\max\{0,\erisk(\fe,\setFr)-b\}\big]
  \\
   &= \int_0^\infty \P\{\erisk(\fe,\setFr) > b+t\} \, dt
  \le \int_0^\infty e^{-(n\,t/c)^{1/m}} \, dt
    = \frac{m!\,c}{n}
  \,,
\end{split} \end{equation}
where the expectation is taken with respect to the random sample $\Dn$
and the potential extra randomness of the estimator $\hn$
producing $\fe = \hn(\Dn)$.

\subsection{Bounded losses}
\label{sec:ubound-bnd}

We start with a simple case when the loss function $\loss$ is bounded,
which implies that $|\Zrisk(f,\fr)| \le B$ holds for some $B > 0$.
Now notice that the random variable in the exponent of \cond{mom},
that is $r\E[\Zrisk(f,\fr)] - \Zrisk(f,\fr)$,
has a negative expected value $(r-1)\E[\Zrisk] < 0$.
Furthermore, $\E[\Zrisk(f,\fr)]$ is bounded away from zero
as $\E[\Zrisk(f,\fr)] > \radFo/n$ by the definition of $\setFor$.
Then, combining these observations
with Hoeffding's lemma, we get the following result:
\begin{lemma}
  Suppose that $\sup_{f\in\setFor}|\Zrisk(f,\fr)| \le B$
  \as~holds for some $B > 0$.
  Then $(\dr,\loss,\setFor,\fr)$ satisfies \cond{mom}
  with any $r \in (0,1)$ and $\moms \ge nB^2/(2(1-r)\radFo)$.
\end{lemma}
\begin{proof}
  Fix any $f \in \setFor$, and set $\Zf \defeq \Zrisk(f,\fr)$.
  Then, apply Hoeffding's lemma to the bounded random variable $\Zf$
  to get
  \begin{equation*}
    \E\big[e^{(r/\moms)\E[\Zf] - (1/\moms)\Zf}\big]
      = e^{(r-1)\E[\Zf]/\moms} \, \E\big[e^{(\E[\Zf]-\Zf)/\moms}\big]
    \le e^{(r-1)\E[\Zf]/\moms + B^2/(2\moms^2)}
    \le 1
    \,,
  \end{equation*}
  as $\frac{(r-1)\E[\Zf]}{\moms} + \frac{B^2}{2\moms^2}
        < \E[\Zf]\big(\frac{(r-1)}{\moms} + \frac{n}{2\radFo\moms^2}\big)
      \le 0$
  by $1 < \frac{n}{\radFo}\E[\Zf]$ due to the definiton of $\setFor$
  and $\E[\Zf] = \erisk(f,\fr)$.
\end{proof}

However, as $\moms$ scales with $n/\radFo$,
one should choose $\radFo > 0$ to balance the appropriate terms
of \cref{thm:ermub}. This can be achieved by setting
$\radFo = \Theta\big(B\sqrt{n\,\HeF(\epsilon,\setFor)}\big)$
which balances $\frac{\moms}{n}\HeF(\epsilon,\setFor)
                = \Theta\big(\frac{B^2}{\radFo}\HeF(\epsilon,\setFor)\big)$
with $\radFo/n$.
This way the bound of \cref{thm:ermub} scales with
$B\sqrt{\HeF(\epsilon,\setFor)/n}$,
which cannot be improved in general
\citep[e.g.,][Section~5]{BartlettBousquetMendelson2005}.

\subsection{Unbounded losses}
\label{sec:ubound-unbnd}

We show that the moment condition \cond{mom} is often implied
by the \emph{Bernstein condition}
\citep[e.g.,][Definition~1.2]{LecueMendelson2013},
which is said to be satisfied by $(\dr,\loss,\setF,\fr)$
if there exists $\Bc > 0$ such that for all $f \in \setF$, we have
\begin{equation}
\label{eq:Bernstein}
  \E\big[\Wf^2\big] \le \Bc \, \E\big[\Zrisk(f,\fr)\big]
  \,, \quad
  \Wf \defeq f(\vX)-\fr(\vX)
  \,.
\end{equation}

Then, \cref{thm:cond-mom} shows that the Bernstein condition implies
\cond{mom} when $\Zrisk(f,\fr)$ can be decomposed
to the \subexponential random variables, $\Zrisk(f,\fr)/\Wf$ and $\Wf$.
\begin{lemma}
\label{thm:cond-mom}
  Let $\radFo > 0$
  and suppose that $(\dr,\loss,\setFor,\fr)$ satisfies
  the Bernstein condition~\eqref{eq:Bernstein} with some $\Bc > 0$.
  Furthermore,
  suppose that $\sup_{f\in\setFor}\onorm[p]{\Zrisk(f,\fr)/\Wf} \le R$
  and $\sup_{f\in\setFor}\onorm{\Wf} \le B$
  with some $B, R > 0$ and $p, q \ge 1$ having $\frac1p+\frac1q \le 1$.
  Then $(\dr,\loss,\setFor,\fr)$ satisfies \cond{mom}
  for any $\momr \in (0,1)$ and
  $\moms \ge 4t\Kn^{2/\min\{p,q\}}
             \max\big\{\frac{4R^2\Bc}{(t-1)(1-r)},BR\big\}$,
  where $t > 1$ is arbitrary, and
  $\Kn \defeq 4\ln\big(4\min\big\{\sup_{f\in\setFor}\Kz^{1/4}[\Wf],
                                  \frac{nBR}{\radFo}\big\}\big)$.
\end{lemma}
\begin{proof}
  Let $\Zf \defeq \Zrisk(f,\fr)$,
  $\Vf \defeq \Zrisk(f,\fr)/\Wf$, and fix any $f \in \setFor$.
  Then by the definition of $\setFor$, the Cauchy-Schwartz inequality,
  and \crefi{thm:onorm}{mom} with $s = 2$, we get
  \[ \frac{\radFo}{n}
             \le \E[\Zf] \le \E\big[|\Vf \Wf|\big]
             \le \E\big[\Vf^2\big]^{\frac12}
                 \, \E\big[\Wf^2\big]^{\frac12}
             \le 2 R \, \E\big[\Wf^2\big]^{\frac12}
     \,, \]
  which implies $\E\big[\Wf^2\big] \ge \big(\radFo / (2nR)\big)^2$.
  Combining this with
  \crefi{thm:onorm}{mom} for $s = 4$, we obtain
  $\Kz[\Wf] \le \E\big[\Wf^4\big]\big(\radFo/(2nR)\big)^{-4}
            \le (4nBR/\radFo)^4$.
  Then, by using \cref{thm:onorm-mom} with
  $\W \leftarrow \Wf$, $\Z \leftarrow \Vf$, and
  $2\ln(64\,\Kz[\Wf]) \le 2\Kn$, $z \defeq \min\{p,q\}^{-1}$,
  we get for all $2 \le k \in \setN$ that
  \begin{equation*} \begin{split}
    \E\big[|\Zf|^k\big]
      = \E\big[|\Vf \Wf|^k\big]
    \le (k!/2) \big(16 \Kn^{2z} \, \E\big[\Wf^2\big] R^2\big)
               \big(4\Kn^{2z} BR\big)^{k-2}
    \,.
  \end{split} \end{equation*}
  Hence, the conditions of Bernstein's lemma (\cref{thm:Bernstein})
  hold for $\Zf$ and $\theta \ge 4t \Kn^{2z} B R$ with any $t > 1$,
  so by $(1-4\Kn^{2z} B R / \theta)^{-1} \le t/(t-1)$, we obtain
  \begin{equation*} \begin{split}
    \E\big[e^{(\momr/\moms)\E[\Zf]-(1/\moms)\Zf}\big]
    &= e^{(\momr-1)\E[\Zf]/\moms}\,\E\big[e^{(\E[\Zf]-\Zf)/\moms}\big]
    \\
    &\le \exp\Big(\frac{\momr-1}{\moms}\E[\Zf]
                  + \frac{16t\Kn^{2z} R^2}{(t-1)\moms^2}\E[\Wf^2]\Big)
    \\
    &\le \exp\bigg(\frac{\E[\Zf]}{\theta}
                   \Big(\momr-1 + \frac{16t\Kn^{2z} R^2 \Bc}
                                       {(t-1)\moms}\Big)\bigg)
     \le 1
    \,,
  \end{split} \end{equation*}
  where in the last step we applied the Bernstein condition
  \eqref{eq:Bernstein} also implying
  $\E[\Zf] \ge 0$ for all $f \in \setFor$,
  and used $\moms \ge \frac{16t\Kn^{2z} R^2 \Bc}{(t-1)(1-r)}$.
\end{proof}

Notice that \cref{thm:cond-mom} ``splits'' the \subexponential property
of the random variable $\Zrisk(f,\fr)$ between $\Zrisk(f,\fr)/\Wf$ and $\Wf$.
For the squared loss $\loss = \losssq$, using $p = q = 2$
provides the \subgaussian setting \citep{LecueMendelson2013}.
Furthermore, when the random variable $\Wf$ is bounded,
that is $|\Wf| \le B$~\as, we have $\onorm{\Wf} \le 2B$ for all $q \ge 1$,
hence we can use $p = 1$ and $q = \infty$ to cover the setting of
uniformly bounded functions and \subexponential noise with the squared loss
\citep[Section~9.2]{VanDeGeer2000}.
For bounded problems
\citep[e.g.,][]{BartlettMendelson2006,Koltchinskii2011},
when $\Zrisk(f,\fr)$ is bounded,
we can use $p = q = \infty$, and eliminate the $\Kn$ term completely.

\cref{thm:cond-mom} also shows that even in the worst case
we can set the leading constant in the bound of \cref{thm:ermub}
as $\theta = \Ordo\big(\Kn^2\,\Bc\big)$
with $\Kn^2 = \Ordo\big(\ln^2(n/\radFo)\big)$,
which scales logarithmically in the sample size $n$,
and depends on the regression parameters only through
the Bernstein condition \eqref{eq:Bernstein}
of $(\mu,\loss,\setFor,\fr)$.
In the following sections we investigate this dependence
for a few popular regression settings.

\subsubsection{Lipschitz losses}
\label{sec:ubound-lip}

Observe that if $\sup_{f\in\setFor}\onorm{\Wf} \le B$ holds
for some $B > 0$,
then the Bernstein condition is always satisfied for the function class
$\setFor$ with any $\radFo > 0$ by $\Bc = 2n B^2 / \radFo$.
To see this, use \crefi{thm:onorm}{mom} with $s = 2$,
and $1 < (n/\radFo)\erisk(f,\setFr)$ for any $f \in \setFor$
due to the definition of $\setFor$, to obtain
\begin{equation*}
  \E\big[\Wf^2\big] \le 2B^2
                      < \frac{2n B^2}{\radFo} \erisk(f,\setFr)
                      = \frac{2n B^2}{\radFo} \E[\Zrisk(f,\fr)]
  \,.
\end{equation*}

When the loss function $\loss$ is $R$-Lipschitz in its second argument
with some $R > 0$, that is
$|\loss(y,\hy_1)-\loss(y,\hy_2)| \le R\,|\hy_1-\hy_2|$
for all $y, \hy_1, \hy_2 \in \setR$,
we clearly have $|\Zrisk(f,\fr)| \le R\,|\Wf|$ for any $f\in\setF$.
Then, the requirements of \cref{thm:cond-mom} hold
with $R$, $\Bc = 2nB^2/\radFo$, and any $\radFo > 0$,
so we obtain \cond{mom}.

As $\Bc$ and so $\moms$ scale with $n/\radFo$,
this setting is similar to the bounded case (\cref{sec:ubound-bnd}),
so we choose $\radFo > 0$ to balance the appropriate terms
of \cref{thm:ermub}. For this, here we use
$\radFo = \Theta\big(BR\sqrt{n\,\Kn\HeF(\epsilon,\setFor)}\big)$
which balances $\frac{\moms}{n}\HeF(\epsilon,\setFor)
                = \Theta\big(\frac{\Kn(BR)^2}
                                  {\radFo}\HeF(\epsilon,\setFor)\big)$
with $\radFo/n$.
This way the bound of \cref{thm:ermub} scales with
$BR\sqrt{\Kn\HeF(\epsilon,\setFor)/n}$,
which again cannot be improved in general.%
\footnote{For example,
  consider estimating the mean of a standard \gaussian random variable
  through constant functions using the absolute value loss,
  and derive the optimal $\Ordo(n^{-1/2})$ rate by \cref{thm:ermub}.}

\subsubsection{Strongly-convex losses}
\label{sec:ubound-strong}

Now consider a loss function $\loss$, which is
\emph{$\scl$-strongly convex} in its second argument, that is
$ \loss(y,\lambda \hy_1 + (1-\lambda)\hy_2) \le
  \lambda\loss(y,\hy_1) + (1-\lambda)\loss(y,\hy_2)
  - \frac{\scl \lambda (1-\lambda)}{2} |\hy_1-\hy_2|^2 $
holds for all $y, \hy_1, \hy_2 \in \setR$ and $\lambda \in (0,1)$.
Then, if $\risk(\fr) \le \risk\big((f+\fr)/2\big)$ is satisfied
for all $f \in \setFo$,
the Bernstein condition \eqref{eq:Bernstein} holds with $C = 4/\scl$.
To see this, proceed similarly to
\citet[Lemma~7]{BartlettJordanMcAuliffe2006}
by using the strong convexity property of $\loss$
to get for all $f \in \setFo$ that
\begin{equation} \begin{split}
\label{eq:Bernstein-scl}
  \E\big[\Wf^2\big]
  &\le (4/\scl) \Big(\risk(f)+\risk(\fr)-2\risk\big((f+\fr)/2\big)\Big)
  \\
  &\le (4/\scl) \big(\risk(f)-\risk(\fr)\big)
   = (4/\scl) \E\big[\Zrisk(f,\fr)\big]
  \,.
\end{split} \end{equation}
Notice that the condition
$\risk(\fr) \le \inf_{f\in\setF}\risk\big((f+\fr)/2\big)$
is implied by the definition of $\fr$ if
either $\fr$ is a regression function
defined by the reference class $\setFr = \setFXR$,
or when $\fr \in \argmin_{f\in\setF}\risk(f)$
and $\setF$ is midpoint convex,
that is $f,\fr \in \setF$ implies $(f+\fr)/2 \in \setF$.%

Then, we need $\onorm{\Wf} \le B$
and $\onorm[p]{\Zrisk(f,\fr)/\Wf} \le R < \infty$
to satisfy the requirements of \cref{thm:cond-mom}.
Again, we get the latter for any Lipschitz loss
as in \cref{sec:ubound-lip}. However, here the constant $\Bc$ of the
Bernstein condition \eqref{eq:Bernstein} does not scale with the sample
size $n$, which provides better rates by \cref{thm:ermub}.

One such example is logistic regression with $\Y \in (0,1)$ \as~%
using the cross-entropy loss $\loss = \lossce$ where
$\lossce(y,\hy) \defeq y\ln(y/\hy)
                         + (1-y)\ln\big((1-y)/(1-\hy)\big)$
for $y, \hy \in (0,1)$,
and a hypothesis class 
$\setFo \subseteq \{\setX \to [\lambda,1-\lambda]\}$
with some $\lambda \in (0,1/2)$.
Because the $\lossce$ function is $1/\lambda$-Lipschitz
and $(1-\lambda)^{-2}$-strongly convex in its second argument
over the domain $(0,1) \times [\lambda,1-\lambda]$,
we get the requirements of \cref{thm:cond-mom}
with $p = q = \infty$, $B = 1-\lambda$, $R = 1/\lambda$,
and $\Bc = 4(1-\lambda)^2$ by \eqref{eq:Bernstein-scl}.%
\footnote{To get these values, use
          $\partial_z\lossce(y,z) = \frac{z-y}{z(1-z)}$ and
          $\partial_{zz}\lossce(y,z) = \frac{y}{z^2}+\frac{1-y}{(1-z)^2}$.}
Here notice that $\Bc$ does not scale with the sample size $n$
as for the general Lipschitz case in \cref{sec:ubound-lip},
which allows \cref{thm:ermub} to deliver better rates.

Notice that the squared loss $\loss = \losssq$
is $2$-strongly convex, however, it is not Lipschitz over the real line.
Fortunately, this is not needed for the condition
$\onorm[p]{\Zrisk(f,\fr)/\Wf} \le R$ which holds for some $R > 0$
when $\onorm[p]{\Y-\fr(\vX)}$ is bounded due to the decomposition
$\Zrisk(f,\fr)/\Wf = \Wf+2(\fr(\vX)-\Y)$.

In the following section, we combine these observations
for the squared loss $\loss = \losssq$
with \cref{thm:ermub,thm:cond-mom}
to provide a detailed analysis for linear least squares estimation.

\section{Linear least squares regression}
\label{sec:linear}

Here we provide an analysis for the \emph{linear least squares regression}
setting, which uses the squared loss
and considers ERM estimators over affine hypothesis classes
for regression problems with \subgaussian distributions defined as
\begin{equation*} \begin{split}
\label{eq:subgregprbs}
  \setMls
  \defeq \big\{\mu \,\big|\,
                 (\vX,\Y) \sim \mu ,\,
                 \vX \in \setR^d ,\, \Y \in \setR ,\,
                 \onorm[2]{\vX-\E\vX} \le \BX ,\,
                 \osmallnorm[2]{\Y-\E\Y} \le \BY
         \big\}
  \,,
\end{split} \end{equation*}
with some \subgaussian parameters $\BX, \BY > 0$,
and feature space $\setX \defeq \setR^d$ with dimension $d \in \setN$.
Further, we consider affine reference classes
$\setFr \subseteq \setFlinUc
        \defeq \{\vx \mapsto \va^\T\vx + b,\, \vx \in \setR^d\}$,
and use \emph{least squares estimators} (LSEs),
that is ERM estimators \eqref{eq:erm} using the squared loss
$\loss = \losssq$, over some hypothesis class within affine functions
$\setFn \subseteq \setFlinUc$.

First, we derive a general result (\cref{thm:ermub-constrained})
which is specialized later for the slope constrained (\cref{sec:linear-lasso})
and penalized (\cref{sec:linear-ridge}) settings.
For the general result,
we set the reference class to the set of slope-bounded affine functions
as $\setFr \defeq \setFlin$,
where $\setFlin \defeq \{\vx\mapsto\va^\T\vx+b : \enorm{\va} \le L\}$
for some Lipschitz bound $L > 0$.

Here we only consider penalty functions which are independent of the bias term
satisfying $\frac{\partial}{\partial b}\pen(\vx\mapsto\va^\T\vx+b) = 0$.
Then, we have
\begin{equation*}
  \Ybar-\va^\T\vXbar
  = \argmin_{b\in\setR}
    \frac1n\sum_{i=1}^n|\va^\T\vX_i+b-\Y_i|^2
  + \penLtwo(\vx\mapsto\va^\T\vx+b)
  \,,
\end{equation*}
hence any estimate $\fe \in (\err,\penLtwo)$-ERM($\setFlinUc$)
can be expressed as $\fe(\vx) \defeq \va_n^\T(\vx-\vXbar)+\Ybar$
with some $\va_n \in \setR^d$.
Moreover, because
$\E[\Y-\va^\T\vX] = \argmin_{b\in\setR}\E\big[|\va^\T\vX+b-\Y|^2\big]$,
we can also write any reference function $\fr$ as
$\fr(\vx) \defeq \va_*^\T(\vx-\E\vX)+\E\Y$ with some $\va_* \in \setR^d$.

Now introduce the following linear function classes:
\begin{equation*} \begin{split}
  \setFlinmu(t)
    &\defeq \big\{\vx \mapsto \va^\T(\vx-\E\vX)+b
                  : \enorm{\va} \le L,\, b-\E\Y \in [-t,t]
            \big\}
  \,, \\
  \setFlinDn(t)
    &\defeq \big\{\vx \mapsto \va^\T(\vx-\vXbar)+b
                  : \enorm{\va} \le L,\, b-\Ybar \in [-t,t]
            \big\}
  \,,
\end{split} \end{equation*}
for any $t \ge 0$.
Observe that any reference function satisfies
$\fr \in \argmin_{f\in\setFlinmu(t)}\risk(f)$ for any $t \ge 0$,
and any estimate $\fe \in (\err,\pen)$-ERM($\setFlinUc$)
with Lipschitz bound $\enorm{\va_n} \le L$
satisfies $\fe \in (\err,\pen)$-ERM\big($\setFlinDn(t)$\big)
for all $t \ge 0$.

Because distribution $\mu$ is unknown,
estimators cannot be represented by the class $\setFlinmu(t)$,
just by its data-dependent approximation $\setFlinDn(t)$.
However, as the quantities $\E\vX$ and $\E\Y$ are ``well-approximated''
by $\vXbar$ and $\Ybar$ for \subgaussian random variables $\vX$ and $\Y$,
the function classes $\setFlinmu(t)$ and $\setFlinDn(t)$ are ``close''.
More precisely, one can show (see the proof of \cref{thm:ermub-lin})
that
\mbox{$\setFlinmu(0) \subseteq \setFlinDn(\tn) \subseteq \setFlinmu(2\tn)$}
holds with probability at least $1-\gamma$,
where $\tn \defeq \Theta\big(\max\{L\BX,\BY\}\sqrt{\ln(1/\gamma)}\big)$.
Hence, \cref{thm:ermub} is applicable for such function sets
and provides an excess risk upper bound for
$(\err,\pen)$-ERM($\setFlinDn$) estimators.

Next we point out that the Lipschitz bound on the slope
$\enorm{\va_n}$ can be often improved by using the ERM property \eqref{eq:erm}
when the smallest eigenvalue of the feature covariance matrix
is bounded away from zero.
Denote the covariance matrix by
$\fcov \defeq \E\big[(\vX-\E\vX)(\vX-\E\vX)^\T\big]$
and let its smallest eigenvalue be $\BX^2\eigmu \ge 0$.
So we have $\fcov \succeq \BX^2\eigmu\Id_d$, where
$\Id_d$ denotes the $d \times d$ identity matrix.
Then consider \cref{thm:Lref},
which provides a refinement for the Lipschitz bound when $\eigmu > 0$.
\begin{lemma}
  \label{thm:Lref}
  Let $\mu \in \setMls$ be any distribution such that $\eigmu > 0$,
  $\loss = \losssq$ be the squared loss,
  and consider an estimate $\fe \in (\alpha,\beta)$-ERM($\setFlinDn(0)$)
  such that $\alpha \le \frac1n\sum_{i=1}^n|\Y_i-\Ybar|^2$
  and $\pen : \setFlinUc \to [0,\infty)$
  satisfies $\frac{\partial}{\partial b}\pen(\vx\mapsto\va^\T\vx+b) = 0$
  for some $\gamma \in (0,1)$.
  Then, $\P\{\enorm{\va_n} \le \rlip(L,\gamma)\} \ge 1-\gamma$ holds
  for any $n \in \setN$, where
  \[ \rlip(L,\gamma) \defeq
     \min\bigg\{L^2,\, \frac{1}{\eigmu}
                       \bigg(\frac{\BY^2}{\BX^2}
                             + \frac{dL^2}{n}
                       \bigg)\rlipLog(d,n,\gamma,\eigmu)\bigg\}^{1/2}
     \,,
  \]
  and $\rlipLog(d,n,\gamma,\eigmu)
       \defeq 10
              \big(11 \ln(23/\eigmu) \ln(3n/\min\{d,n\}) + 6\big)
              \ln(6/\gamma)$.

  Furthermore, for any $\fr \in \argmin_{f\in\setFlin}\risk(f)$,
  we have $\enorm{\va_*} \le \rlip(L,\gamma)$.
\end{lemma}
\begin{proof}
  Here we only prove the claim for $\va_*$
  and provide the proof for $\va_n$ later in \cref{sec:aux}.

  Using the definition of $\fr$ and the fact that the constant function
  $\vx\mapsto\E\Y$ is in $\setFlin$,
  we have $\risk(\fr) \le \risk(\vx\mapsto\E\Y)$
  which can be rearranged into
  $\E\big[|\va_*^\T(\vX-\E\vX)|^2\big] \le 2\E[\va_*^\T(\vX-\E\vX)(\Y-\E\Y)]$.
  Using this, $2ab \le \frac{a^2}{2}+2b^2$,
  and \crefi{thm:onorm}{mom} with $s = q = 2$, we obtain
  \begin{equation*} \begin{split}
    \BX^2\eigmu \enorm{\va_*}^2
    \le \va_*^\T\fcov\va_*
      = \E\big[|\va_*^\T(\vX-\E\vX)|^2\big]
      \le 4\,\E\big[|\Y-\E\Y|^2\big]
      \le 4\BY^2
      \,,
  \end{split} \end{equation*}
  which proves the claim for $\va_*$ after rearrangement.
\end{proof}

For convenience, define $\rlip(L,\gamma) \defeq L$
for distributions $\mu$ with $\eigmu = 0$.
Then notice that \cref{thm:Lref} improves the Lipschitz bound significantly
when $\eigmu$ is bounded away from zero, $\BY/\BX \ll L$, and $n \gg d$.

Finally, we put all the details together
and provide the following bound on the excess risk
for linear least squares estimation:
\begin{theorem}
\label{thm:ermub-lin}
  Consider any distribution $\mu \in \setMls$,
  the squared loss $\loss = \losssq$,
  and an estimator $\fe \in (\err,\pen)$-ERM($\setFlinUc$)
  with penalty $\pen : \setFlinUc \to [0,\infty)$
  satisfying \mbox{$\frac{\partial}{\partial b}\pen(\vx\mapsto\va^\T\vx+b) = 0$},
  $\P\{\enorm{\va_n} > L\} \le \gamma/2$ for some $L > 0$,
  $\P\big\{\pen(\fr) + \err > \apxB\big\} \le \gamma/16$
  for some $\gamma \in (0,1)$ and reference function $\fr$.
  Then for all $n \in \setN$,
  we have with probability at least $1-\gamma$ that
  \begin{equation*} \begin{split}
    \erisk\big(\fe,\setFlin\big)
    &= \Ordo\bigg(\frac{d\max\{\rlip(L,\gamma)\BX,\BY\}^2}{n}
                  \, \cgamma(n,d,\gamma)
                  + \apxB
            \bigg)
    ,
  \end{split} \end{equation*}
  where 
  $\cgamma(n,d,\gamma) \defeq
                       \big(\ln(e\,n/\min\{d,n\}) + \ln\ln(e/\gamma)
                        \big)\ln(e\,n/\min\{d,n\})\ln(1/\gamma)$.
\end{theorem}
\begin{proof}
  First, condition on the event $\enorm{\va_n} \le L$,
  and use $\cref{thm:Lref}$ with $\gamma \leftarrow \gamma/4$
  to get
  \begin{equation*} \begin{split}
     \P\{\erisk(\fe,\setFlin) > b\}
    &\le \gamma/2
       + \P\big\{\erisk(\fe,\setFlin)
                 \,\ind\{\enorm{\va_n} \le L\} > b\big\}
     \\
    &\le \frac{3\gamma}{4}
               + \P\big\{\erisk\big(\fe,\setFlin\big)
                 \,\ind\{\enorm{\va_n} \le \hL\} > b\big\}
     \,,
  \end{split} \end{equation*}
  for any $b > 0$, where $\hL \defeq \rlip(L,\gamma/4)$,
  and $\ind\{\cdot\}$ denotes the indicator function.
  To find an appropriate $b$ for the second term,
  we will apply \cref{thm:ermub}.

  Fix $\tn \defeq t_0 \sqrt{\ln(32/\gamma)}$
  with $t_0 \defeq 2\max\{\hL\BX,\BY\}$,
  and define the function sets
  $\setFi \defeq \setFlinmuL{\hL}(0)$,
  $\setFn \defeq \setFlinDnL{\hL}(\tn)$,
  and $\setFo \defeq \setFlinmuL{\hL}(2\tn)$.
  Notice that $\fe \in (\err,\pen)$-ERM($\setFn$)
  holds conditioned on the event $\enorm{\va_n} \le \hL$.
  Moreover, $\fr \in \setFi$ due to \cref{thm:Lref},
  so we have \cond{bound}.

  To prove \cond{enclose} for $\setFi$, $\setFn$, and $\setFo$,
  use $\osmallnorm[2]{\hL\esmallnorm{\vXbar-\E\vX}+|\Ybar-\E\Y|} \le t_0$
  due to \crefi{thm:onorm}{norm}, and \crefi{thm:onorm}{prob}, to get
  \begin{equation*} \begin{split}
      \P\big\{\setFi \subseteq \setFn \subseteq \setFo\big\}
    \ge 1 - \P\big\{\hL\esmallnorm{\vXbar-\E\vX}+|\Ybar-\E\Y| > \tn\big\}
    \ge 1 - 2 e^{-\tn^2/t_0^2}
      = 1 - \gamma/16
    \,.
  \end{split} \end{equation*}

  Next we use \cref{thm:cond-mom} to show \cond{mom}.
  For this, write $f \in \setFo$ as $f(\vx) \defeq \va^\T(\vx-\E\vX)+b$,
  and observe that
  \begin{equation*}
    |\Wf| = |f(\vX)-\fr(\vX)|
          = \big| (\va-\va_*)^\T(\vX-\E\vX) + b-\E\Y \big|
          \le 2\hL\enorm{\vX-\E\vX} + 2\tn
    \,.
  \end{equation*}
  Now pick $f \in \setFo$ arbitrarily,
  and use $(a+b)^2 \le 2a^2+2b^2$ with Jensen's inequality,
  to show that $\Wf$ is \subgaussian, that is
  \begin{equation*}
    \E\big[e^{\Wf^2/\Rtn^2}\big]
    \le \E\big[e^{8\hL^2\enorm{\vX-\E\vX}^2/\Rtn^2}\big]
        \, e^{8\tn^2/\Rtn^2}
    \le 2 
    \,,
  \end{equation*}
  with $\Rtn \defeq 4\max\{\hL\BX,\tn\}$,
  so we have $\onorm[2]{\Wf} \le \Rtn$.
  Additionally, as $\fr \in \argmin_{f\in\setFo}\risk(f)$
  due to \cref{thm:Lref}, and the function set $\setFo$ is convex,
  the requirements of \cref{thm:cond-mom} are satisfied
  by \eqref{eq:Bernstein-scl} and
  \[
    \onorm[2]{\Zrisk(f,\fr)/\Wf}
    = \osmallnorm[2]{(\va+\va_*)^\T(\vX-\E\vX)+2(\E[\Y]-\Y)}
    \le 2(\hL\BX+\BY) \defeq R
    \,,
  \]
  so we obtain \cond{mom} with $\radFo = \min\{d,n\}BR$,
  $r = 1/2$, and any $\theta = \Omega\big(\Kn \max\big\{\Rtn,\BY\}^2\big)$,
  where \mbox{$\Kn = \Ordo\big(\ln(e\,n/\min\{d,n\})\big)$}.

  Next, to prove \cond{lip},
  write $g \in \setFo$ as $g(\vx) \defeq \hat{\va}^\T(\vx-\E\vX)+\hat{b}$
  with $\enorm{\hat{a}} \le \hL$ and $|\hat{b}-\E\Y| \le 2\tn$.
  Then, using the Cauchy-Schwartz inequality, we obtain
  \begin{equation*} \begin{split}
    \Zrisk(f,g)
   &= \big(f(\vX)-g(\vX)\big)
      \big(f(\vX)+g(\vX)-2\Y\big)
   \\
   &= \big((\va+\hat\va)^\T(\vX-\E\vX)+b+\hat{b}-2\Y\big)
      \big((\va-\hat\va)^\T(\vX-\E\vX)+b-\hat{b}\big)
   \\
   &\le \Big(2\hL\enorm{\vX-\E\vX}+4\tn+2|\Y-\E\Y|\Big)
        \left[\begin{array}{c}
          3\hL\enorm{\vX-\E\vX} \\
          4\tn \\
        \end{array}\right]^\T
        \left[\begin{array}{c}
          \enorm{\va-\hat\va}/(3\hL) \\
          |b-\hat{b}|/(4\tn) \\
        \end{array}\right]
   \\
   &\le \underbrace{\Big(\big(2\hL\enorm{\vX-\E\vX}+4\tn+2|\Y-\E\Y|
                         \big)\ln^{-1/2}(32/\gamma)\Big)^2}_%
                   {\defeq\,\Lmod(\vX,\Y)}
        \distF(f,g)
   \,,
  \end{split} \end{equation*}
  where $\distF(f,g) \defeq 
         \ln(32/\gamma)
         \big(\esmallnorm{\va-\hat\va}^2/(2\hL)^2
              + |b-\hat{b}|^2/(4\tn)^2\big)^{1/2}$
  is a metric on $\setFo$.
  As the radius of $\setFo$ under~$\distF$ is bounded by $\ln(32/\gamma)$,
  that is $\sup_{f\in\setFo}\distF(f,\vx\mapsto\E\Y) \le \ln(32/\gamma)$,
  we have by \cref{thm:He-bound-par}
  that $\HeF(\epsilon,\setFo)
        \le (d+1)\big(\ln(3/\epsilon)+\ln\ln(32/\gamma)\big)$
  for all $\epsilon \in (0,3]$.

  Further, as $\osmallnorm[2]{\sqrt{\Lmod(\vX,\Y)}}
               \le (2t_0 + 6\tn)/\sqrt{\ln(32/\gamma)}$, 
  we also have $\osmallnorm[1]{\Lmod(\vX,\Y)} \le (8t_0)^2$,
  so \crefi{thm:onorm}{norm} implies
  \mbox{$\onorm[1]{\frac1n\sum_{i=1}^n\Lmod(\vX_i,\Y_i)} \le (8 t_0)^2$}.
  Hence, due to \crefi{thm:onorm}{prob}, we obtain
  $\P\big\{\E[\Lmod(\vX,\Y)] + \frac1n\sum_{i=1}^n\Lmod(\vX_i,\Y_i)
           > \procmodB\big\} \le \gamma/8$
  with $\procmodB \defeq \E[\Lmod(\vX,\Y)] + 64\,t_0^2\ln(32/\gamma)$
  having $\procmodB < 128\,\tn^2$
  by \crefi{thm:onorm}{mom} with $s = 1$
  and the definition of $\tn$.

  Finally, we can apply \cref{thm:ermub}
  with $\gamma \leftarrow \gamma/4$,
  $\delta \defeq \epsilon$ ignoring \cond{psi} with $\oqmod = \infty$,
  and choosing $\epsilon \defeq \min\{d,n\}/n$
  satisfying $\epsilon \in [0,3]$,
  to get with probability at least $1-\gamma/4$ that
  \begin{equation*} \begin{split}
    \erisk(\fe,\setFlin)\,\ind\{\enorm{\va_n} \le \hL\}
    &\le 2\bigg(\frac{\moms\HeF(\epsilon,\setF)}{n}
                + 16\epsilon\,T + \apxB
          \bigg) + \frac{\radFo}{n}
    \\
    &= \Ordo\bigg(\frac{d\theta}{n}
                  \Big(\ln(e\,n/\min\{d,n\})+\ln\ln(e/\gamma)\Big)
                 + \frac{\tn^2}{n}
                 + \apxB\bigg)
    ,
  \end{split} \end{equation*}
  which proves the claim by
  $\moms = \Ordo(\ln(e\,n/\min\{d,n\})\,\tn^2)$,
  $\tn^2 = \Ordo\big(\max\{\hL\BX,\BY\}^2 \ln(1/\gamma)\big)$,
  and $\radFo = \Ordo(t^2)$.
\end{proof}

\subsection{Linear least squares with quadratic slope constraint}
\label{sec:linear-lasso}

Here we specialize \cref{thm:ermub-lin}
to train an affine LSE with $\enorm{\cdot}$-bounded slope
without using any penalty term ($\pen \defeq 0$),
and provide the following result:
\begin{corollary}
\label{thm:ermub-constrained}
  Consider any distribution $\mu \in \setMls$,
  the squared loss $\loss = \losssq$,
  and an estimate $\fe \in \err$-ERM($\setFlinDn$)
  with any $\err$
  having $\err = \Ordo\big(\frac1n\sum_{i=1}^n|\Y_i-\Ybar|^2 / n\big)$
  with probability at least $\gamma/16$
  for some $\gamma \in (0,1)$.
  Then for all $n \in \setN$,
  we have with probability at least $1-\gamma$ that
  \begin{equation*} \begin{split}
    \erisk(\fe,\setFlin)
   &= \Ordo\bigg(\frac{d\max\{\rlip(L,\gamma)\BX,\BY\}^2}{n}
                 \, \cgamma(n,d,\gamma)
           \bigg)
    .
  \end{split} \end{equation*}
\end{corollary}
\begin{proof}
  Notice that
  \begin{equation}
  \label{eq:YYbar}
    \frac1n\sum_{i=1}^n|\Y_i-\Ybar|^2
    = \frac1n\sum_{i=1}^n|\Y_i-\E\Y|^2 - |\Ybar-\E\Y|^2
    \quad \Rightarrow \quad
    \onorm[1]{\frac1n\sum_{i=1}^n|\Y_i-\Ybar|^2} \le \BY^2
    \,,
  \end{equation}
  so we have $\P\{\alpha > \apxB\} \le \gamma/16$ by \crefi{thm:onorm}{prob}
  for $\apxB = \Ordo(\frac{\BY^2\ln(1/\gamma)}{n})$.
  Then, the claim follows from \cref{thm:ermub-lin} with $\pen = 0$
  using $\P\{\pen(\fr)+\err > \apxB\} = \P\{\alpha > \apxB\} \le \gamma/16$.
\end{proof}

Notice that \cref{thm:ermub-constrained} provides an $\Ordo(1/n)$
bound on the excess risk up to logarithmic factors
for any problem in $\setMls$
without any further statistical assumptions
such as the bounded magnitude of the kurtosis $\sup_{f\in\setFlin}\Kz[\Wf]$
as discussed in \cref{sec:review}.

Furthermore, the bound of \cref{thm:ermub-constrained}
after integrated by \eqref{eq:experisk}
is comparable to the conjecture of \citet{Shamir2015}
stating that ERM estimates achieve optimal expected excess risk
up to logarithmic factors for bounded distributions.
Our bound is only slightly weaker in general than the conjecture
by scaling with $d\max\{\rlip(L,\gamma)\BX,\BY\}^2
                 \le d\max\{L\BX,\BY\}^2$
instead of $\max\{(L\BX)^2,d\BY^2\}$.
However, if either $\BY = \Omega(L\BX)$,
or $\eigmu = \Omega(1)$ and $n = \Omega(d^2)$,
then \cref{thm:ermub-constrained} matches the bound of the conjecture
up to logarithmic factors.

\subsection{Linear least squares with quadratic slope penalty}
\label{sec:linear-ridge}

Now we drop the fixed Lipschitz bound on the estimators,
and use the reference class \mbox{$\setFr = \setFlinUc$}.
Then, for $\fr \in \argmin_{\setFlinUc}\risk(f)$,
we write $\fr(\vx) = \va_*^\T(\vx-\E\vX)+\E\Y$
and set $\Ls \defeq \enorm{\va_*}$.%
\footnote{If there are multiple $\fr$,
          we choose the one with the smallest slope $\enorm{\va_*}$.}
Then consider the \emph{ridge regression} \citep{HoerlKennard1970} estimate
$\rife \in (0,\penLtwo)$-ERM($\setFlinUc$)
using the quadratic penalty term
$\penLtwo(\vx\mapsto\va^\T\vx+b) \defeq \ripar\enorm{\va}^2$
with some $\ripar \ge 0$.
This estimator can be also computed in closed-form as
$\rife(\vx) \defeq \riva^\T(\vx-\vXbar)+\Ybar$,
\begin{equation}
\label{eq:ridge-slope}
  \riva \defeq \Big(\ripar\Id_d +
                    \frac{1}{n}\sum_{i=1}^n(\vX_i-\vXbar)(\vX_i-\vXbar)^\T
               \Big)^{-1}
               \Big(\frac{1}{n}\sum_{i=1}^n(\vX_i-\vXbar)\Y_i\Big)
  \,.
\end{equation}

It is known that minimizing the empirical risk without any slope restriction
(i.e.~$\ripar = 0$) might result in infinite expected excess risk
\citep[Example~3.5]{HuangSzepesvari2014}.
Moreover, \cref{thm:ermub-constrained} with $p = 2$ implies
(by Lagrangian relaxation and $L \leftarrow \Ls$)
that for each distribution $\mu \in \setMls$,
there exists $\ripar \ge 0$ such that the excess risk rate
of ridge regression is bounded by $\Ordo(1/n)$.
In this section, we are interested in choosing $\ripar$
independently of the parameters of $\mu$.

For this, \cref{thm:lsenorm} with \eqref{eq:ridge-slope}
provides an upper bound for the Lipschitz factor of $\rife$ as
$\enorm{\riva} \le \sqrt{\frac{1}{4n}\sum_{i=1}^n|\Y_i-\Ybar|^2/\ripar}$
implying
$\P\big\{\enorm{\va_n}^2 > \BY^2\ln(4/\gamma)/(4\ripar)\big\}
 \le \gamma/2$
due to~\eqref{eq:YYbar} and \crefi{thm:onorm}{prob}.
Then $L = \max\{\Ls^2,\BY^2\ln(4/\gamma)/(4\ripar)\}^{1/2}$
provides a common Lispchitz bound for $\rife$, $\fr$,
and $\fr \in \argmin_{f\in\setFlin}\risk(f)$.
Hence, we can apply \cref{thm:ermub-lin} and get the following result:
\begin{corollary}
\label{thm:ermub-penalized}
  Consider any problem $\mu \in \setMls$,
  the squared loss $\loss = \losssq$,
  and an estimate $\rife \in (0,\penLtwo)$-ERM($\setFlinUc$)
  with some $\ripar > 0$.
  Then for all $\gamma \in (0,1)$, any $n \in \setN$,
  and $L_\ripar \defeq \max\{\Ls^2,\BY^2\ln(4/\gamma)/(4\ripar)\}^{1/2}$,
  we have with probability at least $1-\gamma$ that
  \begin{equation*} \begin{split}
    \erisk(\rife,\setFlinUc)
    &= \Ordo\bigg(\Big(\frac{d}{n}
                       \max\big\{\rlip(L_\ripar,\gamma) \BX,\BY\big\}^2
                       + \ripar \Ls^2
                  \Big) \, \cgamma(n,d,\gamma)
            \bigg)
    .
  \end{split} \end{equation*}
\end{corollary}
\begin{proof}
  As explained above, we have $\P\{\enorm{\va_n} > L_\ripar\}\le \gamma/2$.
  Then, the claim follows directly from \cref{thm:ermub-lin}
  using $L = L_\ripar$,
        $\fr \in \argmin_{f\in\setFlin}\risk(f)$,
        $\frac{\partial}{\partial b}
         \penLtwo(\vx\mapsto\va^\T\vx+b) = 0$,
        $\err = 0$,
    and $\penLtwo(\fr) = \ripar \Ls^2$.
\end{proof}

Notice that
$\rlip(L_\ripar,\gamma)
 \le \max\big\{\Ls^2, \BY^2\ln(4/\gamma)/(4\ripar)\big\}^{1/2}$,
hence for $n \ge d$ and $\ripar = \sqrt{d/n}$,
\cref{thm:ermub-penalized} upper bounds
the excess risk $\erisk(\rife,\setFlinUc)$
with probability at least $1-\gamma$
by $\Ordo\big(\max\{\Ls,\BY\}^2 \, (1+\BX)^2 \sqrt{d/n}
              \, \ln^3(1/\gamma)\big)$
up to logarithmic factors.

If $\eigmu = \Omega(1)$, we can also use
$\rlip^2(L_\ripar,\gamma)
 \le \frac1\eigmu
     \big(\frac{\BY^2}{\BX^2}
          + \frac{d\max\{\Ls^2,\BY^2\ln(4/\gamma)/(4\ripar)\}}{n}\big)
     \rlipLog(d,n,\gamma,\eigmu)$.
Then, for $n \ge d$ and $\ripar = d/n$,
\cref{thm:ermub-penalized} upper bounds
the excess risk $\erisk(\rife,\setFlinUc)$
with probability at least $1-\gamma$
by $\Ordo\big(\max\{\Ls,\BY\}^2 \, (1+\BX)^2 d/n
              \, \ln^3(1/\gamma)\big)$
up to logarithmic factors.

Finally, we point out that \cref{thm:ermub-penalized} is comparable to
the result of \citet[Remarks~4 and~12]{HsuKakadeZhang2014},
which uses similar conditions to prove the same rates in terms of $d$ and~$n$
with slightly better constants,
but only for the bounded setting when $\enorm{\vX} \le \BX$~\as~holds.

\section{Proof of the upper bound}
\label{sec:uboundprf}

In this section, we finally prove our main result, \cref{thm:ermub},
our upper bound on the excess risk of ERM estimators.

The strategy is to ``remove'' the data-dependence of the hypothesis class
$\setFn$ by \cond{enclose}, and decompose the excess risk to ``supremal'' and
``approximation'' error terms. Then we reduce the former to a general
concentration inequality (\cref{thm:eproc-sup-n})
with \cond{lip}, \cond{psi}, \cond{mom}, and
upper bound the latter using the ERM property~\eqref{eq:erm} and \cond{bound}.

To work with probabilistic arguments, we use the following rule
without further notice:
$\P\{\W+\Z > t+s\} \le \P\{\W > t\} + \P\{\Z > s\}$
for any random variables $\W, \Z$, and $t, s \in \setR$,
which holds as a simple corollary of the law of total probability.

Use \cond{enclose}, and the definition of $\setFor$
implying $\erisk(f,\setFr) \le \radFo/n$
for all $f \in \setF \setminus \setFor$,
to get with probability at least $1 - \gamma/4$ that
\begin{equation*} \begin{split}
  \erisk(\fe,\setFr)
  = \erisk(\fe,\setFr) \, \ind\{\setFi \subseteq \setFn \subseteq \setFo\}
  \le \erisk(\fe,\setFr) \,
      \ind\big\{\setFi \subseteq \setFn,\, \fe \in \setFor\big\}
    + \frac{\radFo}{n}
  \,.
\end{split} \end{equation*}
Then notice that the ERM property~\eqref{eq:erm} implies
$\eriskn(\fe,\fr) \le \inf_{f\in\setFn}\eriskn(f,\fr) + \pen(f) + \err$,
so we can transform the previous excess risk inequality
with $\erisk(\fe,\setFr) = \erisk(\fe,\fr)$, and \cond{bound},
to get with probability at least $1-\gamma/2$ that
\begin{equation} \begin{split}
\label{eq:uboundprf:2}
  \erisk(\fe,\setFr)
 &\le \Big(\erisk(\fe,\fr) - \frac{\eriskn(\fe,\fr)}{r}
                           + \frac{\eriskn(\fe,\fr)}{r}\Big)
      \, \ind\{\setFi \subseteq \setFn,\, \fe \in \setFor\}
      + \frac{\radFo}{n}
  \\
 &\le \frac{1}{r}\bigg(
      \sup_{f\in\setFor} \big\{r\erisk(f,\fr) - \eriskn(f,\fr)\big\}
    + \inf_{f\in\setFi} \eriskn(f,\fr) + \pen(f) + \err
      \bigg)
    + \frac{\radFo}{n}
  \\
 &\le \frac{1}{r} 
      \sup_{f\in\setFor} \procl(f,\Dn) + \frac{\apxB}{r} + \frac{\radFo}{n}
  \,,
\end{split} \end{equation}
where $\procl(f,\Dn) \defeq r\erisk(f,\fr) - \eriskn(f,\fr)$.
It remains to bound the supremal term,
for which we use the general concentration inequality \cref{thm:eproc-sup-n}
by showing that its conditions \eqref{thm:eproc-sup-n:lip},
                               \eqref{thm:eproc-sup-n:psi},
                               \eqref{thm:eproc-sup-n:mom}
are satisfied by \cond{lip}, \cond{psi}, \cond{mom}, respectively.

Recall that $\erisk(f,\fr) = \E[\Zrisk(f,\fr)]$
and $\eriskn(f,\fr) = \frac1n\sum_{i=1}^n\Zriski(f,\fr)$.
Notice that $\procl(f,\Dn)$ can be rewritten for any $f$ as
$\procl(f,\Dn) = r\E\big[\Zrisk(f,\fr)\big]
                         - \frac{1}{n}\sum_{i=1}^n\Zriski(f,\fr)$.
Define $\proc(f,\Dn) \defeq \procl(f,\Dn) - \E[\procl(f,\Dn)]$,
and observe that $\E[\procl(f,\Dn)] = (r-1)\E[\Zrisk(f,\fr)]$, so
$\proc(f,\Dn) = \E\big[\Zrisk(f,\fr)\big]
                - \frac{1}{n}\sum_{i=1}^n\Zriski(f,\fr)$.
Then, we also have for all $f, g \in \setFor$ that
\begin{equation} \begin{split}
\label{eq:uboundprf:1}
  \proc(f,\Dn) - \proc(g,\Dn)
  = \E\big[\Zrisk(f,g)\big]-\frac{1}{n}\sum_{i=1}^n\Zriski(f,g)
  \,.
\end{split} \end{equation}
Because $\proc(f,\Dn)-\proc(g,\Dn)$ is the sum of $n$ independent, centered
random variables $\E[\Zrisk(f,g)] - \Zriski(f,g)$ with
$\onorm{\E[\Zrisk(f,g)]-\Zriski(f,g)} \le \oqmod \, \distF(f,g)$
according to \cond{psi}, \cref{thm:onorm-sum-indept} implies that
$\onorm{\proc(f,\Dn)-\proc(g,\Dn)} \le (4 \oqmod \, n^{-1/2}) \, \distF(f,g)$
for any $f, g \in \setFor$.
Hence, $\proc$ and $\distF$ satisfy \crefi{thm:eproc-sup-n}{psi}
with $\W \leftarrow \Dn$, $\oqmod \leftarrow 4\oqmod\,n^{-1/2}$,
and $\setF \leftarrow \setFor$.

Next, using \cond{lip}, we can upper bound \eqref{eq:uboundprf:1} \as~by
$\procmod(\Dn)\,\distF(f,g)$, where
\[ \procmod(\Dn) \defeq \E[\Lmod(\vX,\Y)]
                      + \frac{1}{n}\sum_{i=1}^n\Lmod(\vX_i,\Y_i)
   \,. \]
Hence, $\proc$, $\distF$, and $\procmod$ satisfies
\crefi{thm:eproc-sup-n}{lip}
with $\procmodB$ as given by \cond{lip}
and $\gamma \leftarrow \gamma/2$.

Finally, using the \iid~property of the sample $\Dn$ and \cond{mom},
we have for any $f \in \setFor$ that
\begin{equation} \begin{split}
\label{eq:uboundprf:3}
  \E\Big[e^{\procl(f,\Dn) / (\moms/n)}\Big]
  = \prod_{i=1}^n
    \E\left[e^{(\momr/\moms)\E[\Zrisk(f,\fr)]
               - (1/\moms)\Zriski(f,\fr)}\right]
  \le 1
  \,.
\end{split} \end{equation}
So $\procl$ satisfies \crefi{thm:eproc-sup-n}{mom}
with $\moms \leftarrow \moms/n$.

Hence, all the requirements of \cref{thm:eproc-sup-n} hold,
and we get with probability at least $1-\gamma/2$ that
\begin{equation*}
  \sup_{f\in\setFor}\procl(f,\Dn)
  \le \moms\frac{\HeF(\epsilon,\setFor) + \ln(\frac4\gamma)}{n}
    + \frac{32\oqmod}{\sqrt{n}}\int_\delta^\epsilon \hspace{-1mm}
      \sqrt[q]{2\HeF(z,\setFor) + \ln\big(\frac{32\,\epsilon}{z\gamma}\big)}
      \, dz
    + 8 \delta \procmodB
  \,.
\end{equation*}
Combining this with \eqref{eq:uboundprf:2} proves the claim.

\subsection{Suprema of empirical processes}
\label{sec:supempproc}

In this section, our goal is to prove \cref{thm:eproc-sup-n},
which we used in \cref{sec:uboundprf} as the main tool
to prove \cref{thm:ermub}.
For this, we start with finite class lemmas,
then adapt the classical chaining argument
(e.g., \citealp[Section~3]{Pollard1990};
       \citealp[Chapter~3]{VanDeGeer2000};
       \citealp[Section~13.1]{BoucheronLugosiMassart2012})
to our setting, and finally put these together
to prove \cref{thm:eproc-sup-n}.

First, consider the probabilistic version of the well-known inequality
about the maximum of finitely many random variables
(e.g., \citealp[Lemma~7]{BianchiLugosi1999};
       \citealp[Theorem~2.5]{BoucheronLugosiMassart2012}).
\begin{lemma}
\label{thm:eproc-max-sqrtn}
  Let $\setF$ be a nonempty, finite set (that is $1 \le |\setF| < \infty$),
  $\sigma \in [0,\infty)$,
  and $\W_f$ be random variables such that $\onorm{\W_f} \le \sigma$
  holds for all $f \in \setF$.
  Then for all $\gamma > 0$,
  \mbox{$\P\big\{\max_{f\in\setF}\W_f > \sigma\sqrt[q]{\ln(2|\setF|/\gamma)}\big\}
        \le \gamma$}.
\end{lemma}
\begin{proof}
  The claim is trivial for $\sigma = 0$.
  Let $\sigma > 0$
  and set $t \defeq \sigma\sqrt[q]{\ln(2|\setF|/\gamma)}$.
  Then, using the union bound and \crefi{thm:onorm}{prob}, we get
  \begin{equation*}
    \P\Big\{\max_{f\in\setF}\W_f > t\Big\}
    \le \sum_{f\in\setF} \P\big\{\W_f > t\big\}
    \le 2|\setF| e^{-t^q/\sigma^q}
      = \gamma
    \,.
    \qedhere
  \end{equation*}
\end{proof}

When a moment condition, similar to \cond{mom},
is satisfied for $\W_f$,
\cref{thm:eproc-max-sqrtn} can be strengthened by the following result
(for further explanation, see the discussion after \cref{thm:eproc-max-n}).
\begin{lemma}
\label{thm:eproc-max-n}
  Let $\setF$ be a nonempty, finite set (that is $1 \le |\setF| < \infty$),
  $\moms \in (0,\infty)$,
  and $\W_f$ be random variables such that $\E\big[e^{\W_f/\moms}\big] \le 1$
  holds for all $f \in \setF$.
  Then for all $\gamma > 0$,
  $\P\big\{\max_{f\in\setF}\W_f > \moms\ln(|\setF|/\gamma)\big\} \le \gamma$.
\end{lemma}
\begin{proof}
  Set $t \defeq \moms \ln(|\setF|/\gamma)$.
  Then, using the union and Chernoff bounds, we get
  \begin{equation*}
    \P\Big\{\max_{f\in\setF}\W_f > t\Big\}
    \le \sum_{f\in\setF} \P\big\{\W_f > t\big\}
    \le |\setF| \, e^{-t/\moms} \, \E\big[e^{\W_f/\moms}\big]
    \le |\setF| \, e^{-t/\moms}
      = \gamma
    \,.
    \qedhere
  \end{equation*}
\end{proof}

To see that \cref{thm:eproc-max-n} is indeed stronger than
\cref{thm:eproc-max-sqrtn} for our purposes,
notice that they scale differently in their parameters
$\moms$ and $\sigma$ when applied to averages of
independent random variables. If $\W_f^{(1)},\ldots,\W_f^{(n)}$ are
$n \in \setN$ independent centered random variables
with $\onorm{\W_f^{(i)}} \le \sigma$ for $q \in \{1,2\}$,
their average satisfies
$\onorm{\frac{1}{n}\sum_{i=1}^n\W_p^{(i)}} \le 4\sigma/\sqrt{n}$
by \cref{thm:onorm-sum-indept}.
On the other hand, it is straightforward to show
(as we did by Equation~\ref{eq:uboundprf:3})
that if $\W_f^{(1)},\ldots,\W_f^{(n)}$
are $n$ independent (not necessarily centered) random variables with
$\E\big[\exp(\W_f^{(i)}/\moms)\big] \le 1$,
then their average satisfies the moment condition with $\moms/n$.
This speed-up, from $\frac{\sigma}{\sqrt{n}}\ln|\setF|$ to
$\frac{\moms}{n}\ln|\setF|$, will allow us
to derive better bounds when the moment condition \cond{mom} holds
with $\moms$ being independent of $n$.

We now extend \cref{thm:eproc-max-sqrtn} to infinite classes
by a probabilistic version of the standard chaining argument.
The proof goes along the development of Lemma~3.4 of
\citet{Pollard1990}, replacing the packing sets by internal covering numbers
(for better numerical constants) and the sample continuity condition by
uniform Lipschitzness (for truncating the integral at $\delta$). The result
is also similar to Proposition~3 of \citet{BianchiLugosi1999}, which works
for the \subgaussian case, provides and expected value result,
and uses external covering numbers with a slightly
different chaining argument.
\begin{theorem}
\label{thm:eproc-sup-sqrtn}
  Let $(\setF,\distF)$ be a separable metric space,
  $\W$ be a random variable on some set~$\setW$,
  and $\proc : \setF \times \setF \to \setR$ be a function.
  Furthermore, suppose that the following conditions hold:
  \begin{enumerate}[label=(\alph*),ref=\alph*]
  \itemsep0em
  \item \label{thm:eproc-sup-sqrtn:lip}
        there exists $\gamma \in (0,1)$, $\procmodB \ge 0$,
        and $\procmod : \setW \to [0,\infty)$
        such that $\P\{\procmod(\W) > \procmodB\} \le \gamma/2$,
        and $\proc(f,\W)-\proc(g,\W) \le \distF(f,g) \, \procmod(\W)$ \as
        for all $f, g \in \setF$,
  \item \label{thm:eproc-sup-sqrtn:psi}
        $\onorm{\proc(f,\W)-\proc(g,\W)} \le \oqmod \, \distF(f,g)$
        with some $\oqmod \in [0,\infty]$
        for all $f, g \in \setF$,
  \item \label{thm:eproc-sup-sqrtn:bound}
        there exist $\beta \ge 0$ and $f_0 \in \setF$ such that
        $\proc(f_0,\W) = 0$~\as,
        and $\sup_{f\in\setF} \distF(f,f_0) \le \beta$.
  \end{enumerate}
  Then, for all $0 \le \delta \le \beta/2$,
  we have with probability at least $1-\gamma$ that
  \begin{equation*}
    \sup_{f\in\setF}\proc(f,W)
    \le 4\oqmod \int_\delta^{\beta/2}
                \sqrt[q]{\HeF(z,\setF) + \ln(4\beta/(z\gamma))} \, dz
      + 4\delta T
    \,.
  \end{equation*}
  Furthermore, the result holds without \eqref{thm:eproc-sup-sqrtn:psi},
  that is using $\oqmod = \infty$, $\delta = \beta/2$, and $\infty\cdot0 = 0$.
\end{theorem}
\begin{proof}
  If there exists $z \in (\delta,\beta/2]$ such that $\NcF(z,\setF) = \infty$,
  then the integral is infinite and so the claim is trivial.
  The claim is also trivial for $\beta = 0$ or $\oqmod = 0$.
  Now assume that $0 < \beta$, $\oqmod \in (0,\infty)$
  and $\NcF(z,\setF) < \infty$ for all $z \in (\delta,\beta/2]$.

  Let $\delta > 0$
  and $m \in \setN \cup \{0\}$ be such that $2\delta \le \beta/2^m < 4\delta$.
  Further, let $\setF_0 \defeq \{f_0\}$,
               $\epsilon_0 \defeq \beta$, $\epsilon_k \defeq \beta/2^k$,
  and $\setF_k$ be an $\epsilon_k$-cover of $\setF$ under $\distF$
  with minimal cardinality for all $k \in \{1,\ldots,m\}$.
  Notice that $\setF_0$ is an $\epsilon_0$-cover
  by \eqref{thm:eproc-sup-sqrtn:bound},
  and define $\gkf$ to be the closest element to $f \in \setF$ in $\setF_k$
  for all $k = 0,\ldots,m$, that is
  $\gkf \in \argmin_{g\in\setF_k}\distF(f,g)$.

  Fix some $k \in \{0,\ldots,m-1\}$ and $f\in\setF_{k+1}$.
  When $k = 0$, we have
  $\distF(f,\gkf) = \distF(f_0,f) \le \beta = \epsilon_0$
  by \eqref{thm:eproc-sup-sqrtn:bound},
  while for $k > 0$, $\distF(f,\gkf) \le \epsilon_k$ holds by
  the definition of~$\setF_k$.
  So \eqref{thm:eproc-sup-sqrtn:psi} implies that
  $\onorm{\proc(f,\W) - \proc(\gkf,\W)} \le \epsilon_k \oqmod$.
  Then, we can chain maximal inequalities for all $k = 0,\ldots,m-1$ by
  using $\P\{\X+\Y > t\}
         \le \P\{\X > t-t_0\} + \P\{\Y > t_0\}$
  for any random variables $\X, \Y$,
  and \cref{thm:eproc-max-sqrtn} with
  $b_k \defeq \epsilon_k \oqmod \sqrt[q]{\ln(2|\setF_{k+1}|/\gamma_k)}$
  and $\gamma_k \defeq 2^{-(k+2)} \gamma$,
  to get for all $t_k > 0$ that
  \begin{equation} \begin{split}
  \label{eq:eproc-sup-sqrtn:1}
   &\P\Big\{\max_{f\in\setF_{k+1}}\proc(f,\W) > t_k\Big\}
    = \P\Big\{\max_{f\in\setF_{k+1}}
              \proc(\gkf,\W) + \proc(f,\W) - \proc(\gkf,\W) > t_k\Big\}
    \\ & \hspace{10mm}
    \le \P\Big\{\max_{f\in\setF_k}\proc(f,\W) > t_k-b_k\Big\}
      + \P\Big\{\max_{f\in\setF_{k+1}}\proc(f,\W)-\proc(\gkf,\W) > b_k\Big\}
    \\ & \hspace{10mm}
    \le \P\Big\{\max_{f\in\setF_k}\proc(f,\W) > t_k-b_k\Big\} + \gamma_k
    \,.
  \end{split} \end{equation}
  Additionally, using \eqref{thm:eproc-sup-sqrtn:lip},
  $\sup_{f\in\setF}\distF(f,\gmf) \le \epsilon_m < 4\delta$,
  and $\P\{\procmod(\W) > T\} \le \gamma/2$,
  we get for all $t > 0$ that
  \begin{equation} \begin{split}
  \label{eq:eproc-sup-sqrtn:2}
   &\P\Big\{\sup_{f\in\setF}\proc(f,\W) > t\Big\}
    = \P\Big\{\max_{f\in\setF_m}
              \proc(\gmf,\W) + \proc(f,\W) - \proc(\gmf,\W) > t\Big\}
    \\ & \hspace{10mm}
    \le \P\Big\{\max_{f\in\setF_m}\proc(f,\W) > t - 4\delta T\Big\}
      + \P\Big\{\sup_{f\in\setF}\proc(f,\W)-\proc(\gmf,\W) > 4\delta T\Big\}
    \\ & \hspace{10mm}
    \le \P\Big\{\max_{f\in\setF_m}\proc(f,\W) > t - 4\delta T\Big\}
      + \P\Big\{\sup_{f\in\setF}\distF(f,\gmf)\,\procmod(\W) > 4\delta T\Big\}
    \\ & \hspace{10mm}
    \le \P\Big\{\max_{f\in\setF_m}\proc(f,\W) > t - 4\delta T\Big\}
      + \gamma/2
    \,.
  \end{split} \end{equation}
  Then, set $t \defeq 4\delta T + \sum_{k=0}^{m-1} b_k$,
  and use \eqref{eq:eproc-sup-sqrtn:1} repeatedly
  with $k = m-1,m-2,\ldots,0$ and $t_k \defeq \sum_{l=0}^k b_l$
  to bound the first term of \eqref{eq:eproc-sup-sqrtn:2}
  with $\P\{\max_{f\in\setF_0}\proc(f,\W) > b_0\}
        = 0$
  and $\sum_{k=0}^{m-1}\gamma_k
       = \frac\gamma4\sum_{k=0}^{m-1}2^{-k}
       < \gamma / 2$,
  which provides with probability at least $1-\gamma$ that
  \begin{equation} \begin{split}
  \label{eq:eproc-sup-sqrtn:3}
    \sup_{f\in\setF}\proc(f,\W)
   &\le 4\delta T +
        \oqmod \sum_{k=0}^{m-1}
        \epsilon_k \sqrt[q]{\ln(2|\setF_{k+1}|/\gamma_k)}
    \\
   &= 4\delta T +
      4\oqmod \sum_{k=0}^{m-1}
      \epsilon_{k+2} \sqrt[q]{\HeF(\epsilon_{k+1},\setF)
                              + \ln(4\beta/(\gamma\epsilon_{k+1}))}
    \,,
  \end{split} \end{equation}
  because $\epsilon_k = 4\epsilon_{k+2}$,
  $\ln|\setF_{k+1}| = \HeF(\epsilon_{k+1},\setF)$,
  and $2/\gamma_k = 4\beta/(\gamma\epsilon_{k+1})$.

  Next, to upper bound the sum of \eqref{eq:eproc-sup-sqrtn:3}
  by an integral, use the nondecreasing property of $\HeF(z,\setF)$
  as $z \to 0$, and $\epsilon_{k+1}-\epsilon_{k+2} = \epsilon_{k+2}$
  to obtain
  \begin{equation*}
    \epsilon_{k+2}\sqrt[q]{\HeF(\epsilon_{k+1},\setF)
                           +\ln(4\beta/(\gamma\epsilon_{k+1}))}
    \le \int_{\epsilon_{k+2}}^{\epsilon_{k+1}}
        \sqrt[q]{\HeF(z,\setF) + \ln(4\beta/(z\gamma)} \, dz
    \,,
  \end{equation*}
  for all $k = 0,\ldots,m-1$ with $\epsilon_{m+1} \defeq \beta/2^{m+1}$.
  Plugging this into \eqref{eq:eproc-sup-sqrtn:3},
  we get the claim by $\epsilon_1 = \beta/2$ and $\epsilon_{m+1} \ge \delta$.

  Finally notice that for $\delta = \beta/2$ (that is $m = 0$),
  we use only \eqref{eq:eproc-sup-sqrtn:2}
  and ignore \eqref{thm:eproc-sup-sqrtn:psi} altogether,
  hence justifying the $0\cdot\infty = 0$ convention
  for the $S = \infty$ case.
  Furthermore, the $\delta = 0$ case can be obtained
  through the limit $\delta \to 0$.
\end{proof}

Now we extend the improved finite class lemma (\cref{thm:eproc-max-n})
to infinite classes and prove the main result of this section,
\cref{thm:eproc-sup-n}.
The idea behind the proof is to apply \cref{thm:eproc-max-n}
in the first step of the chain and the previously developed
chaining technique (\cref{thm:eproc-sup-sqrtn}) to the remainder.
\begin{theorem}
\label{thm:eproc-sup-n}
  Let $(\setF,\distF)$ be a separable metric space,
  $\W$ be a random variable on some set~$\setW$,
  and $\procl : \setF \times \setF \to \setR$ be a function.
  Furthermore, define $\proc(f,w) \defeq \procl(f,w) - \E[\procl(f,\W)]$
  for all $f \in \setF$, $w \in \setW$, and suppose
  that the following conditions hold:
  \begin{enumerate}[label=(\alph*),ref=\alph*]
  \itemsep0em
  \item \label{thm:eproc-sup-n:lip}
        there exists $\gamma \in (0,1)$, $\procmodB \ge 0$,
        and $\procmod : \setW \to [0,\infty)$
        such that $\P\{\procmod(\W) > \procmodB\} \le \gamma/2$,
        and $\proc(f,\W)-\proc(g,\W) \le \distF(f,g) \, \procmod(\W)$ \as
        for all $f, g \in \setF$,
  \item \label{thm:eproc-sup-n:psi}
        $\onorm{\proc(f,\W)-\proc(g,\W)} \le \oqmod \, \distF(f,g)$
        with some $\oqmod \in [0,\infty]$
        for all $f, g \in \setF$,
  \item \label{thm:eproc-sup-n:mom}
        there exists $\moms > 0$ such that
        $\E\big[\exp\big(\procl(f,\W)/\moms\big)\big] \le 1$
        for all $f \in \setF$.
  \end{enumerate}
  Then, for all $0 \le \delta \le \epsilon$,
  we have with probability at least $1-\gamma$ that
  \begin{equation*}
    \sup_{f\in\setF} \procl(f,\W)
    \le \moms\big(\HeF(\epsilon,\setF)+\ln(2/\gamma)\big)
    + 8 \oqmod \hspace{-1mm} \int_\delta^\epsilon \hspace{-1mm}
        \sqrt[q]{2\,\HeF(z,\setF) + \ln(16\epsilon/(z\gamma))} \, dz
      + 8\delta \, \procmodB
    \,.
  \end{equation*}
  Furthermore, the result holds without \eqref{thm:eproc-sup-n:psi},
  that is using $\oqmod = \infty$, $\delta = \beta/2$, and $\infty\cdot0 = 0$.
\end{theorem}
\begin{proof}
  Fix $0 < \delta \le \epsilon$.
  When $\NcF(z,\setF) = \infty$ for some $z \in (\delta,\epsilon]$,
  the claim is trivial, so we can assume that $\NcF(z,\setF) < \infty$
  for all $z \in (\delta,\epsilon]$.

  Let $\setF_\epsilon$ be an $\epsilon$-cover of $\setF$
  under $\distF$ with minimal cardinality
  and define $\gf$ to be the closest element to $f \in \setF$
  in $\setF_\epsilon$, that is
  $\gf \in \argmin_{g\in\setF_\epsilon} \distF(f,g)$.
  Due to Jensen's inequality and \eqref{thm:eproc-sup-n:mom},
  we have $\E[\procl(f,\W)] \le 0$ for all $f \in \setF$.
  Define $\maxgf \in \argmax_{g\in\setF : \distF(g,\gf) \le \epsilon}
                     \E[\procl(g,\W)]$.%
  \footnote{If such $\maxgf$ element does not exist, one can choose another
            element which is arbitrary close to the supremum
            and shrink the gap to zero at the end of the analysis.}
  Then, due to $\distF(\gf,f) \le \epsilon$,
  we have $\E[\procl(f,\W)] \le \E[\procl(\maxgf,\W)]$
  for all $f \in \setF$.
  Further, $\distF(f,\maxgf) \le \distF(f,\gf) + \distF(\gf,\maxgf)
            \le 2\epsilon$
  so $\setF_\epsilon^* \defeq \big\{\maxgf : \gf \in \setF_\epsilon\big\}$
  is a $2\epsilon$-cover of $\setF$ under $\distF$
  with $|\setF^*_\epsilon| = |\setF_\epsilon| = \NcF(\epsilon,\setF)$.

  Now, for the first step of the chain, consider the following decomposition,
  \begin{equation} \begin{split}
  \label{eq:eproc-sup-n:1}
    \sup_{f\in\setF} \procl(f,\W)
    &= \sup_{f\in\setF}
       \Big\{ \procl(\maxgf,\W) + \procl(f,\W) - \procl(\maxgf,\W) \Big\}
    \\
    &= \sup_{f\in\setF}
       \Big\{ \procl(\maxgf,\W) + \proc(f,\W) - \proc(\maxgf,\W)
              + \E\big[\procl(f,\W) - \procl(\maxgf,\W)\big] \Big\}
    \\
    &\le \max_{g\in\setF_\epsilon^*} \procl(g,\W)
       + \sup_{f\in\setF} \Big\{ \proc(f,\W) - \proc\big(\maxgf,\W\big) \Big\}
    \,.
  \raisetag{1.0\baselineskip}
  \end{split} \end{equation}
  Then by \cref{thm:eproc-max-n} and \eqref{thm:eproc-sup-n:mom},
  we obtain with probability at least $1-\gamma/2$ that
  \begin{equation}
  \label{eq:eproc-sup-n:2}
    \max_{g\in\setF_\epsilon^*} \procl(g,\W)
    \le \moms \ln\big(2|\setF_\epsilon^*|/\gamma\big)
      = \moms \big(\HeF(\epsilon,\setF) + \ln(2/\gamma)\big)
    \,.
  \end{equation}

  The rest of the proof is about to upper bound
  the supremal term on the right side of~\eqref{eq:eproc-sup-n:1},
  that is $\sup_{f\in\setF}\proc(f,\W)-\proc(\maxgf,\W)$,
  by using the chaining result of \cref{thm:eproc-sup-sqrtn}.

  Let $\setK \defeq \big\{ (f,\maxgf) : f\in\setF \big\}
             \subseteq \setF \times \setF_\epsilon^*$
  and choose $f_0 \in \argmax_{f\in\setF_\epsilon^*}\E[\procl(f,\W)]$
  so that $f_0 = g_{f_0}^*$ (since $\distF(g_{f_0},f_0) \le \epsilon$),
  implying $(f_0,f_0) \in \setK$.
  Additionally, define
  \begin{equation*} \begin{split}
    \tilde\proc\big((f,\maxgf),w\big)
      &\defeq \proc(f,w) - \proc(\maxgf,w)
    \,, \\
    \tilde\distF\big((f,\maxgf),(h,\maxgh)\big)
      &\defeq \min\big\{ \distF(f,h) + \distF(\maxgf,\maxgh)
                         ,\, 4\epsilon \big\}
    \,,
  \end{split} \end{equation*}
  for all $(f,\maxgf), (h,\maxgh) \in \setF \times \setF_\epsilon^*$,
  and $w \in \setW$.
  Now notice that $(\setF \times \setF_\epsilon^*,\tilde\distF)$
  is a metric space,%
  \footnote{To prove the triangle inequality, use
            $\min\{a+b,c\} \le \min\{a,c\} + \min\{b,c\}$
            for $a,b,c \ge 0$.}
  and for any $f \in \setF$ we have that
  \begin{equation*}
    \tilde\proc\big((f_0,f_0),\W\big) = 0 \,\,\,\as
    \,, \quad
    \tilde\distF\big((f_0,f_0),(f,\maxgf)\big)
    \le 4\epsilon
    \,,
  \end{equation*}
  hence, $\tilde\distF$ and $(f_0,f_0) \in \setK$ satisfies
  \crefi{thm:eproc-sup-sqrtn}{bound} with $\beta = 4\epsilon$.

  Further, since $\distF(f,\maxgf) \le 2\epsilon$ holds
  for all $(f,\maxgf) \in \setK$,
  \eqref{thm:eproc-sup-n:psi}
  and \crefi{thm:onorm}{norm} with $q \ge 1$ implies
  for all $(f,\maxgf), (h,\maxgh) \in \setK$ that
  \begin{equation*} \begin{split}
    \osmallnorm{\tilde\proc((f,\maxgf),\W)
                - \tilde\proc((h,\maxgh),\W)}
    &\le \osmallnorm{\proc(f,\W) - \proc(h,\W)}
       + \osmallnorm{\proc(\maxgh,\W) - \proc(\maxgf,\W)}
    \\
    &\le \big(\distF(f,h) + \distF(\maxgh,\maxgf)\big) \, \oqmod
    \,, \\
    \osmallnorm{\tilde\proc((f,\maxgf),\W)
                - \tilde\proc((h,\maxgh),\W)}
    &\le \osmallnorm{\proc(f,\W) - \proc(\maxgf,\W)}
       + \osmallnorm{\proc(\maxgh,\W) - \proc(h,\W)}
    \\
    &\le 4 \epsilon \, \oqmod
    \,,
  \end{split} \end{equation*}
  hence, $\tilde\proc$ and $\tilde\distF$ satisfies
  \crefi{thm:eproc-sup-sqrtn}{psi} with $\oqmod$.

  Similarly, \eqref{thm:eproc-sup-n:lip} implies that
  \begin{equation*} \begin{split} 
    \tilde\proc\big((f,\maxgf),\W\big) - \tilde\proc\big((h,\maxgh),\W\big)
    &= \proc(f,\W) - \proc(h,\W) + \proc(\maxgh,\W) - \proc(\maxgf,\W)
    \\
    &\le \big(\distF(f,h) + \distF(\maxgh,\maxgf)\big) \, \procmod(\W)
         \,\,\as
    \,, \\
    \tilde\proc\big((f,\maxgf),\W\big) - \tilde\proc\big((h,\maxgh),\W\big)
    &= \proc(f,\W) - \proc(\maxgf,\W) + \proc(\maxgh,\W) - \proc(h,\W)
    \\
    &\le \big(\distF(f,\maxgf) + \distF(\maxgh,h)\big) \, \procmod(\W)
     \le 4 \epsilon \, \procmod(\W)
     \,\,\as
    \,,
  \end{split} \end{equation*}
  so $\tilde\proc$ and $\tilde\distF$ satisfies
  \crefi{thm:eproc-sup-sqrtn}{lip} with $\procmod$.

  Then the requirements of \cref{thm:eproc-sup-sqrtn} hold
  (using $\setF \leftarrow \setK$,
         $\proc \leftarrow \tilde\proc$,
         $\distF \leftarrow \tilde\distF$,
         \mbox{$f_0 \leftarrow (f_0,f_0)$},
         $\beta = 4\epsilon$,
         $\delta \leftarrow 2\delta$,
         $\gamma \leftarrow \gamma/2$),
  so we get with probability at least $1-\gamma/2$ that
  \begin{equation} \begin{split}
  \label{eq:eproc-sup-n:3}
     \sup_{f\in\setF} \big\{ \proc(f,\W) - \proc\big(\maxgf,\W\big) \big\}
   &= \sup_{\kappa\in\setK}\tilde\proc(\kappa,\W)
   \\
   &\le 4 \oqmod \int_{2\delta}^{2\epsilon}
                 \sqrt[q]{\He_{\tilde\distF}(z,\setK)
                          + \ln(32\epsilon/(z\gamma))}
                 \, dz
        + 8\delta \, \procmodB
   \\
   &= 8 \oqmod \int_{\delta}^{\epsilon}
               \sqrt[q]{\He_{\tilde\distF}(2z,\setK)
                        + \ln(16\epsilon/(z\gamma))}
               \, dz
        + 8\delta \, \procmodB
    \,.
  \end{split} \end{equation}

  It remains to bound the entropy of $(\setK,\tilde\distF)$.
  For any $z \in (\delta,\epsilon]$,
  let $\setF_z$ be a \mbox{$z$-cover} of $\setF$ under $\distF$
  with minimal cardinality and define
  $\setK_z \defeq \setF_z \times \setF_\epsilon^*$.
  Then $\setK_z$ is an external $z$-cover of $\setK$
  in the metric space $(\setF \times \setF_\epsilon^*,\tilde\distF)$,
  which means that $\setK_z$ might not be a subset of $\setK$,
  but for any $\kappa \in \setK$ there exists $\hat{\kappa} \in \setK_z$
  for which $\tilde\distF(\kappa,\hat{\kappa}) \le z$.
  Then, as $|\setK_z| = |\setF_z| \cdot |\setF_\epsilon^*|
                    \le \NcF(z,\setF)^2$,
  using the relation between internal and external covering numbers
  \citep[Theorem~1.2.1]{Dudley1999}, we get
  $ \ln\Nc_{\tilde\distF}(2z,\setK)
    \le \ln|\setK_z|
    \le 2\ln\NcF(z,\setF) $.
  Finally, combining this with \eqref{eq:eproc-sup-n:3},
  \eqref{eq:eproc-sup-n:2}, and \eqref{eq:eproc-sup-n:1},
  we get the claim.

  Finally, obtain the $\delta = 0$ case as well
  through the limit $\delta \to 0$.
\end{proof}

\section{Conclusion}
\label{sec:conclusion}

In this paper we provided a probabilistic excess risk upper bound
for ERM estimators in the random design setting.
Although we demonstrated the strength of the result for LSEs,
the result is applicable beyond the squared loss and linear models.

In fact, \cref{thm:ermub} can be used to recover many nonparametric results
in the literature, for example the convex regression result of
\citet[Theorem~1]{Lim2014} for max-affine estimators with $n$ hyperplanes,
or its sieved variant by using only $\lceil n^{d/(d+4)} \rceil$
hyperplanes as discussed by \citet[Theorem~4.2]{BalazsGyorgySzepesvari2015}.
What is interesting though that while we need \cond{psi}
for the entropy integral in the ``standard case'',
the proof for the sieved variant goes similarly to the derivation of
linear regression bounds (\cref{thm:ermub-constrained})
and ignores \cond{psi} completely.
Because of this, the results for the sieved case can be easily
generalized to an unbounded domain (when $\vX$ is \subgaussian)
by \cref{thm:ermub}, but the standard case still needs
a uniformly bounded hypothesis class to satisfy \cond{psi}
without weakening the rate.
This raises the question, whether sieved estimators have this benefit
on the top of  the rate improvement compared to the standard ones, or perhaps
there might appear further improvements or better alternatives to the
chaining technique on the ``tail'' (\cref{thm:eproc-sup-sqrtn}) in the future.


\appendix
\section{Orlicz spaces of random vectors}
\label{sec:random}

In this appendix, we shortly review a few useful properties
of the $\youngfun$-Orlicz norm $\onorm{\cdot}$.
We start with the basic characteristics by \cref{thm:onorm}.
\begin{lemma}
\label{thm:onorm}
  Let $p \ge q \ge 1$,
  and $\vW, \vZ \in \setR^d$ be random vectors
  such that $\onorm{\vW} < \infty$ and $\onorm[p]{\vZ} < \infty$.
  Then the following statements hold:
  \begin{enumerate}[label=(\alph*),ref=\alph*]
  \itemsep0em
  \item \label{thm:onorm:mon}
        $\onorm{\vZ} \le \onorm[p]{\vZ}$,
  \item \label{thm:onorm:norm}
        $\onorm{c\vW} = |c|\onorm{\vW}$, $c\in\setR$,
        $\onorm{\vW+\vZ} \le \onorm{\vW} + \onorm{\vZ}$,
  \item \label{thm:onorm:prob}
        $\P\big\{\enorm{\vW} \ge t\big\}
         \le 2e^{-t^q/\onorm{\vW}^q}$
        for all $t \ge 0$,
  \item \label{thm:onorm:mom}
        $\E\big[\enorm{\vW}^s\big]
         \le 2\big(s/(eq)\big)^{s/q} \onorm{\vW}^s$
        for all $s > 0$.
  \end{enumerate}
\end{lemma}
\begin{proof}
  For \eqref{thm:onorm:mon} observe that $(x,q) \mapsto e^{x^q}$
  is monotone increasing in $q$ for any $x \ge 0$.

  For the first claim of \eqref{thm:onorm:norm}
  simply use the definition of $\onorm{\cdot}$.
  For the second claim, use \eqref{thm:onorm:mon} with $p \ge q$
  and notice that $\youngfun$ is convex by $q \ge 1$.

  For \eqref{thm:onorm:prob}, use the Chernoff bound as
  \[ \P\{\enorm{\vW} > t\}
     \le e^{-t^q/\onorm{\vW}^q} \, \E\Big[e^{\enorm{\vW}^q/\onorm{\vW}^q}\Big]
     \le 2 e^{-t^q/\onorm{\vW}^q}
     \,.
  \]

  For \eqref{thm:onorm:mom},
  use $x^z \le (z/e)^z\,e^x$, $x \ge 0$, $z > 0$
  \citep[in the proof of Lemma~1.4]{BuldyginKozachenko2000}
  by $x \defeq (\enorm{\W}/\onorm{\W})^q$, $z \defeq s/q$,
  and take the expectation of both sides as
  \[ \E\big[\enorm{\vW}^s\big]
     \le \big(s/(eq)\big)^{s/q} \, \E\Big[e^{\enorm{\vW}^q/\onorm{\vW}^q}\Big]
                                   \onorm{\vW}^s
     \le 2 \big(s/(eq)\big)^{s/q} \onorm{\vW}^s
     \,.
     \qedhere
  \] 
\end{proof}
Notice that by \crefi{thm:onorm}{norm},
$\onorm{\cdot}$ is indeed a norm for $q \ge 1$
on the space of random vectors with $\onorm{\vW} < \infty$
(by also using $\onorm{\vW} = 0$ if and only if $\vW = \vzero$).

Next, \cref{thm:onorm-sum-indept} provides a large deviation inequality
for the sum of independent random variables with finite $\youngfun$-Orlicz
norm for the \subexponential ($q = 1$) and \subgaussian ($q = 2$) cases.
\begin{lemma}
\label{thm:onorm-sum-indept}
  Let $q \in \{1,2\}$,
  and $\vW_1, \ldots, \vW_n \in \setR^d$ be independent random variables
  with $\E[\vW_i] = 0$ and $\onorm{\vW_i} < \infty$
  for all $i = 1,\ldots,n$.
  Then,
  \[ \onorm{\sum_{i=1}^n\vW_i}
     \le 4\sqrt[q]{d}\,\bigg(\sum_{i=1}^n\onorm{\vW_i}\bigg)^{1/2}
     \,. \]
\end{lemma}
\begin{proof}
  For $d = 1$ and $q = 1$,
  see Corollary~3.5 of \citet{BuldyginKozachenko2000}.
  For $d = 1$ and $q = 2$,
  combine Lemmas~1.6 and 1.7 of \citet{BuldyginKozachenko2000}
  with the remark at (2.4) of \citet{BoucheronLugosiMassart2012}.
  For $d > 1$ and $q \in \{1,2\}$,
  write $\vW_i = [\W_{i1}\ldots\W_{id}]^\T \in \setR^d$
  for all $i = 1,\ldots,n$,
  and notice that $\norm{\vW_i}_q^q = \sum_{j=1}^d|\W_{ij}|^q$.
  Hence, by using $C \defeq 4\sqrt{\sum_{i=1}^n\onorm{\vW_i}^2}$
  and H\"{o}lder's inequality, we get
  \begin{equation*}
    \E\Big[e^{\norm{\sum_{i=1}^n\vW_i}_q^q/(dC^q)}\Big]
    = \E\Big[e^{\sum_{j=1}^d|\sum_{i=1}^n\W_{ij}|^q/(dC^q)}\Big]
    \le \prod_{j=1}^d\E\Big[e^{|\sum_{i=1}^n\W_{ij}|^q/C^q}\Big]^{1/d}
    \le 2
    \,,
  \end{equation*}
  where the last inequality follows from the $d = 1$ case.
\end{proof}
 
To derive upper bounds for ERM estimators by \cref{thm:cond-mom},
we apply Bernstein's inequality (\cref{thm:Bernstein})
for the product of \subexponential random variables $\W\Z$,
which requires an ``appropriate'' bound on
the higher moments as provided by \cref{thm:onorm-mom}.
Here, ``appropriate'' means that the bound
has to scale with the second moment of one multiplier, say $\W$,
replacing $\onorm{\W}^2$ by $\E[\W^2]$. The price we pay for this is only
logarithmic in the kurtosis $\Kz[\W]$,
which is crucial to our analysis for deriving excess risk upper bound
for problems with highly-skewed distributions
(see \cref{sec:review}).

\begin{lemma}[Bernstein's lemma]
\label{thm:Bernstein}
  Consider a real valued random variable $\W$ satisfying
  $\E\big[|\W|^k\big] \le (k!/2) v^2 c^{k-2}$ for all $2 \le k \in \setN$.
  Then, for all $|s| < 1/c$,
  \begin{equation*}
      \ln \E\Big[e^{s(\E[\W]-\W)}\Big]
      \le \frac{s^2 \, v^2}{2(1 - |s|c)}
      \,.
  \end{equation*}
\end{lemma}
\begin{proof}
  See, for example \citet[Theorem 2.10]{BoucheronLugosiMassart2012}
  with $n \leftarrow 1$,
  and use $X_1 \leftarrow -\W$, $\lambda \leftarrow -s$ when $s < 0$.
\end{proof}

\begin{lemma}
\label{thm:onorm-mom}
  Let $p, q \ge 1$ such that $\frac1p+\frac1q \le 1$, and
  let $\W, \Z \in \setR$ be two random variables such that $\E[\W^2] > 0$,
  \mbox{$\onorm{\W} \le B$}, and $\onorm[p]{\Z} \le R$
  with some $B, R > 0$.
  Then for all $2 \le k \in \setN$, we have
  \begin{equation*}
    \E\big[|\W\Z|^k\big]
    \le (k!/2) \, \E\big[\W^2\big] (2cR)^2 \, \big(c^2 B R\big)^{k-2}
    \,,
  \end{equation*}
  where $c^{\min\{p,q\}} \defeq 2\ln\big(64\,\Kz[\W]\big)$.
\end{lemma}
\begin{proof}
  Let $c > 0$ to be chosen later.
  Then by the Cauchy-Schwartz inequality, we have
  \begin{equation*} \begin{split}
    \E\big[|\W\Z|^k\big]
   &= \E\big[|\W\Z|^k \,\ind\{|\W| \le cB\} \,\ind\{|\Z| \le cR\}\big]
    \\ & \hspace{10mm}
       + \E\big[|\W\Z|^k \,\ind\{|\W| \le cB\} \,\ind\{|\Z| > cR\}\big]
    \\ & \hspace{10mm}
       + \E\big[|\W\Z|^k \,\ind\{|\W| > cB\}\big]
    \\
   &\le \E[\W^2] (cR)^2 (c^2BR)^{k-2}
    \\ & \hspace{10mm}
       + \E[\W^4]^{\frac12} \, (cB)^{k-2} \,
         \E\big[\Z^{2k} \, \ind\{|\Z| > cR\}\big]^{\frac12}
    \\ & \hspace{10mm}
        + \E[\W^4]^{\frac12} \,
          \E\big[\W^{4(k-2)}\Z^{4k}\big]^{\frac14} \,
          \P\{|\W| > cB\}^{\frac14}
    \\
   &\le \E[\W^2]\Big((cR)^2(c^2BR)^{k-2}
    \\ & \hspace{20mm}
                    + \Kz[\W]^{\frac12}(cB)^{k-2}
                      \,\E\big[\Z^{4k}\big]^{\frac14}
                      \,\P\{|\Z| > cR\}^{\frac14}
    \\ & \hspace{20mm}
                    + \Kz[\W]^{\frac12} 
                      \,\E\big[\W^{8(k-2)}\big]^{\frac18}
                      \,\E\big[\Z^{8k}\big]^{\frac18}
                      \,\P\{|\W| > cB\}^{\frac14}
                \Big)
    \,.
  \end{split} \end{equation*}
  Now let $m \defeq \min\{p,q\}$,
  and set $c^m \defeq 4\ln\big(8\,\Kz^{1/2}[\W]\big)$,
  which satisfies $c^m \ge 8$ as $\Kz[\W] \ge 1$
  by Jensen's inequality.
  Then apply \crefi{thm:onorm}{prob}, $2(k/e)^k \le k!$,
  and \crefi{thm:onorm}{mom}, to get
  \begin{equation*} \begin{split}
    \E\big[|\W\Z|^k\big]
   &\le \E[\W^2] (cR)^2 (c^2 B R)^{k-2} \, \cdot
    \\ & \hspace{5mm}
        \bigg(1 + 2^{5/4} \, \Kz[\W]^{\frac12}
                  \Big( \big(\frac{4k}{ep\,c^p}\big)^{\frac{k}{p}}
                        \, e^{-c^p/4}
                      + \big(\frac{8(k-2)}{eq\,c^q}\big)^{\frac{k-2}{q}}
                        \big(\frac{8k}{ep\,c^p}\big)^{\frac{k}{p}}
                        \, e^{-c^q/4}
                  \Big)
        \bigg)
    \\
   &\le \E[\W^2] (cR)^2 (c^2 B R)^{k-2} \,
        \bigg(1+ 2^{5/4} \, \Kz[\W]^{\frac12} \, e^{-c^m/4}
        \Big((k/e)^k + \big(k/e\big)^{k(\frac1q+\frac1p)}\Big)\bigg)
    \\
   &\le \E[\W^2] (cR)^2 (c^2 B R)^{k-2} \,
        \Big(1 + 2^{5/4} \, \Kz[\W]^{\frac12} \, e^{-c^m/4} \, k! \Big)
    \\
   &\le 2 k! \, \E[\W^2] (cR)^2 \, (c^2 B R)^{k-2}
    \,.
  \end{split} \end{equation*}
\end{proof}

\section{Auxiliary tools}
\label{sec:aux}

In this appendix we provide a few auxiliary results
which are used for our proofs.

\begin{lemma}
\label{thm:He-bound-par}
  Let $p \in \setN \cup \{\infty\}$
  and $\setS \subset \setR^d$ with a finite radius under $\norm{\cdot}_p$,
  that is suppose there exists $\vx_* \in \setR^d$ such that
  $\setS \subseteq \{\vx\in\setR^d : \norm{\vx-\vx_*}_p \le R\}$
  for some $R > 0$.
  Then $\He_{\norm{\cdot}_p}(\epsilon,\setS) \le d \ln(3R/\epsilon)$
  for all $\epsilon \in (0,3R]$.
\end{lemma}
\begin{proof}
  By using the volume argument
  \citep[e.g., see][proof of Lemma~4.1]{Pollard1990}.
\end{proof}

%

\begin{lemma}
\label{thm:lsenorm}
  For $A \in \setR^{n \times d}$, $\vb \in \setR^n$ and $r > 0$,
  \mbox{$\enorm{(r\Id_d + A^\T A)^{-1}A^\T\vb}
         \le \frac{\enorm{\vb}}{2\sqrt{r}}$}.
\end{lemma}
\begin{proof}
  Consider a thin singular value decomposition of $A$ given as
  $A = USV^\T$, where $U \in \setR^{n \times d}$ is semi-orthogonal
  ($U^\T U = \Id_d$), $S \in \setR^{d \times d}$ is diagonal with nonegative
  elements, and $V \in \setR^{d \times d}$ is orthogonal
  ($V^\T V = VV^\T = \Id_d$). Then
  \[ \enorm{(r\Id_d+A^\T A)^{-1}A^\T\vb}
     = \enorm{V(r\Id_d+S^2)^{-1}SU^\T\vb}
     \le \enorm{(r\Id_d+S^2)^{-1}S} \enorm{\vb}
     \,,
  \]
  where we used $\enorm{U} = \enorm{V} = 1$
  (as $U^\T U = \Id_d$ implies $\enorm{U\vx}^2 = \enorm{\vx}^2$).
  Finally notice that
  \[ \enorm{(r\Id_d+S^2)^{-1}S}
     \le \max_{s\ge0}\frac{s}{r+s^2}
       = \frac{1}{2\sqrt{r}}
     \,, \]
  which proves the claim.
\end{proof}

\begin{proof}[Proof of \cref{thm:Lref} for $\va_n$]
  \label{prf:Lref:an}
  By the ERM property \eqref{eq:erm} of $\fe$
  and because the constant function
  $\vx\mapsto\Ybar$ is in $\setFlinDn$, we obtain
  $\riskn(\fe) + \pen(\fe)
   \le \riskn(\vx\mapsto\Ybar) + \pen(\vx \mapsto \Ybar) + \err$.
  This can be rearranged into
  $\frac1n\sum_{i=1}^n|\va_n^\T(\vX_i-\vXbar)|^2
   \le \frac2n\sum_{i=1}^n\va_n^\T(\vX_i-\vXbar)(\Y_i-\Ybar) + \err$,
  where we also used $\pen(\cdot) \ge 0$
  and $\pen(\vx \mapsto \Ybar) = 0$.

  Using this, $\fcovh \defeq \frac1n\sum_{i=1}^n
                             (\vX_i-\vXbar)(\vX_i-\vXbar)^\T$,
  and $2ab \le \frac{a^2}{2}+2b^2$, we obtain
  \[ \va_n^\T\fcovh\va_n
     = \frac1n\sum_{i=1}^n|\va_n^\T(\vX_i-\vXbar)|^2
     \le \frac4n\sum_{i=1}^n|\Y_i-\Ybar|^2 + \alpha
     \le \frac5n\sum_{i=1}^n|\Y_i-\Ybar|^2
     \,,
  \]
  which can be transformed to
  $\va_n^\T\fcov\va_n \le
   \frac1r\big(\frac5n\sum_{i=1}^n|\Y_i-\Ybar|^2
               + \va_n^\T(r\fcov-\fcovh)\va_n\big)$
  with any $r \in (0,1]$.
  In particular, with $r \defeq 1/2$, we have
  \begin{equation}
  \label{eq:prf:Lref:an:1}
    \va_n^\T\fcov\va_n \le
    \frac{10}{n}\sum_{i=1}^n|\Y_i-\Ybar|^2 + \W
    \,, \quad \W \defeq \va_n^\T\big((1/2)\fcov-\fcovh\big)\va_n
    \,.
  \end{equation}
  Then, observe that $\frac1n\sum_{i=1}^n|\Y_i-\Ybar|^2
                      = \frac1n\sum_{i=1}^n|\Y_i-\E\Y|^2 - |\Ybar-\E\Y|^2$,
  so by \crefi{thm:onorm}{prob},
  we get with probability at least $1-\gamma/3$ that
  $ \frac{10}{n}\sum_{i=1}^n|\Y_i-\Ybar|^2 \le 10\BY^2\ln(6/\gamma)$.

  Next, notice that
  \begin{equation*} \begin{split}
    \W &= (1/2)\E\big[|\va_n^\T(\vX-\E\vX)|^2\big]
          - \frac1n\sum_{i=1}^n|\va_n^\T(\vX_i-\E\vX)|^2
          + |\va_n^\T(\vXbar-\E\vX)|^2
       \\
       &\le \sup_{\va:\enorm{\va}\le L} \procl(\va)
          + L^2\esmallnorm{\vXbar-\E\vX}^2
    \,,
  \end{split} \end{equation*}
  where $\procl(\va) \defeq
                     (1/2)\E\big[|\va^\T(\vX-\E\vX)|^2\big]
                   - \frac1n\sum_{i=1}^n|\va^\T(\vX_i-\E\vX)|^2$.
  Then, by \cref{thm:onorm-sum-indept}, we have
  $\osmallnorm[2]{\vXbar-\E\vX} \le 4\BX\sqrt{d/n}$,
  so \crefi{thm:onorm}{prob} also implies
  with probability at least $1-\gamma/3$ that
  $L^2\esmallnorm{\vXbar-\E\vX}^2 \le \frac{16d(L\BX)^2}{n}\ln(6/\gamma)$.

  It remains to bound $\sup_{\va:\enorm{\va}\le L}\procl(\va)$
  with probability at least $1-\gamma/3$, for which
  we use $\cref{thm:cond-mom}$ and \cref{thm:eproc-sup-n}.
  First, use $\Y \leftarrow 0$, $\fr \leftarrow (\vx\mapsto0)$,
  $\Wf \leftarrow \va^\T(\vX-\E\vX)$ to satisfy the requirements
  of \cref{thm:cond-mom}
  with $\onorm[2]{\Wf} \le L\BX$, $\radFo \to 0$, $C = 1$, $R = L\BX$
  due to $\Zrisk(f,\fr) = \Wf^2$. Hence, we get
  $\sup_{\va:\enorm{\va}\le L}
   \E\big[e^{\E[|\va^\T(\vX-\E\vX)|^2]/(2\theta)
             -|\va^\T(\vX-\E\vX)|^2/\theta}\big] \le 1$
  with $\moms = 72(L\BX)^2\ln(23/\eigmu)$ by $t \leftarrow 9$
  and $\sup_{\va\ne\vzero}\Kz[\va^\T(\vX-\E\vX)] \le 2/\eigmu^2$
  due to \crefi{thm:onorm}{mom} with $s = 4$ and $q = 2$.
  This further implies
  $\sup_{\va:\enorm{\va}\le L}\E[e^{\procl(\va)/(\theta/n)}] \le 1$
  due to the \iid~property of $\vX,\vX_1,\ldots,\vX_n$,
  and gives us condition \crefi{thm:eproc-sup-n}{mom}
  for all $\va$ having $\enorm{\va} \le L$.

  Furthermore, for any $\va,\hat{\va}$
  with $\enorm{\va} \le L$ and $\enorm{\hat{\va}} \le L$,
  and $\proc(\va) \defeq \procl(\va) - \E[\procl(\va)]$,
  we get by using $a^2-b^2 = (a-b)(a+b)$ that
  \begin{equation*} \begin{split}
    &\proc(\va) - \proc(\hat{\va})
    \\ & \hspace{5mm}
    = \E\Big[|\va^\T(\vX-\E\vX)|^2 - |\hat{\va}^\T(\vX-\E\vX)|^2\Big]
      + \frac1n\sum_{i=1}^n
        |\hat{\va}^\T(\vX_i-\E\vX)|^2 - |\va^\T(\vX_i-\E\vX)|^2
    \\ & \hspace{5mm}
    \le \frac{\esmallnorm{\va-\hat\va}}{L} (2L^2)
        \Big(\E\big[\enorm{\vX-\E\vX}^2\big]
             + \frac1n\sum_{i=1}^n\enorm{\vX_i-\E\vX}^2\Big)
    = \distF(\va,\hat{\va}) \, \procmod(\Dn)
    \,,
  \end{split} \end{equation*}
  with $\distF(\va,\hat\va) \defeq \esmallnorm{\va-\hat\va}/L$
  and $\procmod(\Dn) \defeq 2L^2
        \big(\E\big[\enorm{\vX-\E\vX}^2\big]
             + \frac1n\sum_{i=1}^n\enorm{\vX_i-\E\vX}^2\big)$.
  By \crefi{thm:onorm}{prob} with
  $\osmallnorm[2]{\sqrt{\procmod(\Dn)}} \le 2L\BX$,
  we have $\procmod(\Dn) \le 4(L\BX)^2\ln(12/\gamma)
                         \le 6(L\BX)^2\ln(6/\gamma)$
  with probability at least $1-\gamma/6$.
  Hence, we also have condition \crefi{thm:eproc-sup-n}{lip}.
  Furthermore, by \cref{thm:He-bound-par},
  using $\sup_{\va:\enorm{\va}\le L}\distF(\va,\vzero) \le 1$,
  we also have $\HeF(\epsilon,\{\va:\enorm{\va}\le L\}) \le d\ln(3/\epsilon)$
  for all $\epsilon \in (0,3]$.

  Then, by using \cref{thm:eproc-sup-n}
  with $\epsilon \defeq \delta \defeq 1/n \in (0,3]$
  and $\gamma \leftarrow \gamma/3$,
  we get with probability at least $1 - \gamma/3$ that
  \begin{equation*} \begin{split}
     \sup_{\va:\enorm{\va}\le L} \procl(\va)
    &\le \Big(\frac{\moms \HeF(\epsilon,\{\va:\enorm{\va}\le L\})}{n}
              + 48\epsilon(L\BX)^2
         \Big) \ln(6/\gamma)
    \\
    &\le \frac{(L\BX)^2}{n}\Big(72 d \ln(23/\eigmu)\ln(3n) + 48\Big)
                           \ln(6/\gamma)
    \,.
  \end{split} \end{equation*}
  
  Finally, combining the three probabilistic bounds
  with \eqref{eq:prf:Lref:an:1}, we get
  with probability at least $1-\gamma$ that
  \[ \BX^2\eigmu \enorm{\va_n}^2
     \le \va_n^\T\fcov\va_n
     \le \Big(10\BY^2
              + \frac{8(L\BX)^2}{n}\big(11 d \ln(23/\eigmu) \ln(3n)
                                        + 6\big)
         \Big) \ln(6/\gamma)
     \,,
  \]
  which completes the proof after rearrangement.
\end{proof}

\vskip 0.2in
\bibliography{refs}

\end{document}